\documentclass[sigconf]{acmart}

\usepackage{tikz}
\usepackage{amsfonts}
\usepackage{amsmath}
\usepackage{amsthm}
\usepackage{bm}
\usepackage{booktabs}
\usepackage{graphicx}
\usepackage{caption}
\usepackage{comment}

\usepackage{multirow} 
\usepackage{multicol}

\usepackage{tabularx} 
\usepackage{array, colortbl} 

\newtheorem{theorem}{Theorem}
\newtheorem{definition}{Definition}

\PassOptionsToPackage{table,xcdraw}{xcolor}
\usepackage{xcolor}

\definecolor{TablePurple}{RGB}{151, 107, 185}
\definecolor{TableBlue1}{RGB}{78, 98, 171}   
\definecolor{TableBlue2}{RGB}{70, 158, 180}  
\definecolor{TableGreen}{RGB}{135, 207, 164} 
\definecolor{TableYellow}{RGB}{254, 232, 154} 

\usepackage{algorithm}
\usepackage{algorithmicx}
\usepackage{algpseudocode}
\usepackage[switch]{lineno}

\usepackage{bbding}

\usepackage{enumitem}
\setlist[itemize]{leftmargin=*}

\AtBeginDocument{%
  }

\copyrightyear{2026}
\acmYear{2026}
\setcopyright{cc}
\setcctype{by}
\acmConference[KDD 2026] {Proceedings of the 32nd ACM SIGKDD Conference on Knowledge Discovery and Data Mining V.2}{August 9--13, 2026}{Jeju Island, Republic of Korea.}
\acmBooktitle{Proceedings of the 32nd ACM SIGKDD Conference on Knowledge Discovery and Data Mining V.2 (KDD 2026), August 9--13, 2026, Jeju Island, Republic of Korea}
\acmISBN{979-8-4007-2259-2/2026/08}
\acmDOI{10.1145/3770855.3817770}

\begin{document}

\title{\textsc{Broken Memories}: Detecting and Mitigating Memorization in Diffusion Models with Degraded Generations}


\author{Yuanmin Huang}
\email{yuanminhuang23@m.fudan.edu.cn}
\orcid{0000-0002-4843-5201}
\affiliation{
    \institution{Fudan University}
    \city{Shanghai}
    \country{China}
}

\author{Mi Zhang}
\authornotemark[1]
\email{mi_zhang@fudan.edu.cn}
\orcid{0000-0003-3567-3478}
\affiliation{
    \institution{Fudan University}
    \city{Shanghai}
    \country{China}
}

\author{Chen Chen}
\email{chenc24@m.fudan.edu.cn}
\orcid{0009-0008-1311-2040}
\affiliation{
    \institution{Fudan University}
    \city{Shanghai}
    \country{China}
}

\author{Feifei Li}
\email{ffli23@m.fudan.edu.cn}
\orcid{0009-0006-7868-4272}
\affiliation{
    \institution{Fudan University}
    \city{Shanghai}
    \country{China}
}

\author{Geng Hong}
\email{ghong@fudan.edu.cn}
\orcid{0000-0003-1811-9432}
\affiliation{
    \institution{Fudan University}
    \city{Shanghai}
    \country{China}
}

\author{Xiaoyu You}
\email{xiaoyuyou@ecust.edu.cn}
\orcid{0000-0001-9714-5545}
\affiliation{
    \institution{East China University of Science and Technology}
    \city{Shanghai}
    \country{China}
}

\author{Min Yang}
\authornote{Corresponding author: Mi Zhang and Min Yang.}
\email{m_yang@fudan.edu.cn}
\orcid{0000-0001-9714-5545}
\affiliation{
    \institution{Fudan University}
    \city{Shanghai}
    \country{China}
}

\begin{abstract}
While diffusion models excel at generating high-quality images, their tendency to memorize training data poses significant privacy and copyright risks.
In this work, we for the first time identify that memorization induces internal numerical instability, often manifesting as visually ``broken'' artifacts.
Inspired by stability analysis in numerical methods, we introduce \emph{empirical stability regions} based on latent update norms to quantitatively characterize stable behavior during generation.
Leveraging this, we propose a principled, on-the-fly framework for step-wise detection and adaptive mitigation.
Our approach suppresses memorization without altering prompts or guidance, thereby preserving semantic fidelity and image quality.
Extensive experiments on Stable Diffusion 1.4 demonstrate that our method achieves an AUC $>0.999$ detection performance and a $0.0\%$ memorization rate after mitigation with negligible overhead ($\approx0.01$s per image).
\end{abstract}




\maketitle

\section{Introduction}
Recent advancements in diffusion models have led to breakthroughs in high-quality image generation, showcasing immense potential in creative arts~\cite{rombach2022highresolutionimagesynthesislatent}. 
However, as model scales and training data volumes continue to grow, diffusion models increasingly face a critical issue: \textit{memorized generation}~\cite{somepalli2023understanding}. 
It is evident from recent research that diffusion models inadvertently replicate or highly resemble specific images from their training datasets, leading to potential privacy leakage and copyright infringements~\cite{carlini2023extracting,somepalli2023understanding,chen2024towards}.

\begin{figure}[t]
\centering
\includegraphics[width=\linewidth]{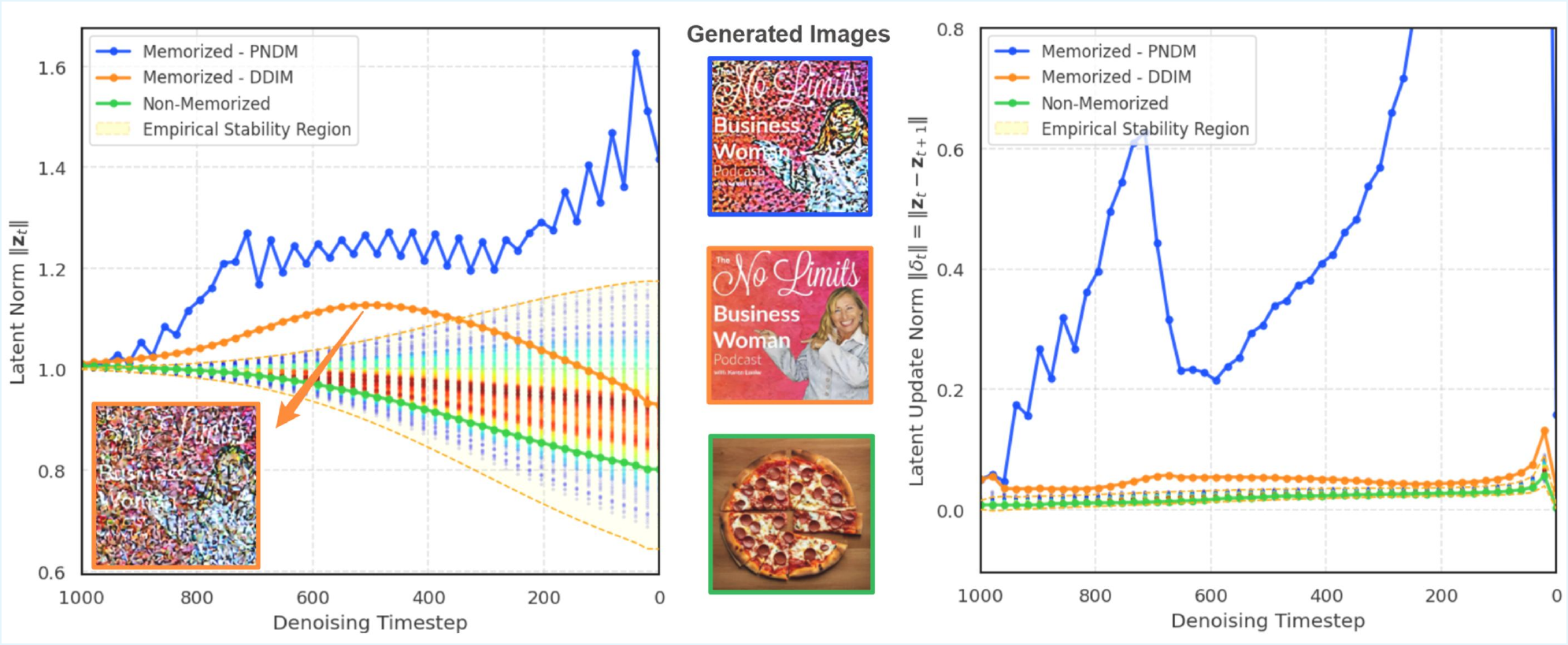}  
\caption{
Memorized generations (blue and orange borders, using PNDM and DDIM) show clear degradation—such as broken details or abnormal intermediate results—compared to non-memorized generations (green border). 
The yellow shaded regions in both plots indicate the empirical stability regions for latent norm ($\|\mathbf{z}_t\|$) and latent update norm ($\|\delta_t\|$). 
Memorized samples consistently deviate from these stability regions during generation, while non-memorized samples stay within them. This motivates our approach: by monitoring and intervening these latent deviations, we can accurately detect and mitigate memorization on the fly.
}
\label{fig:intro}
\end{figure}

Existing research has investigated memorization from multiple perspectives. 
Early studies detect it by comparing generated images with training samples~\cite{somepalli2023understanding}. 
Later works analyze the internal mechanisms underlying memorization, such as attention weights~\cite{ren2024unveiling} or magnitude of conditional guidance directions~\cite{wen2024detecting,chen2025exploringlocalmemorizationdiffusion}. 
To mitigate memorization, these methods propose interventions like modifying prompts~\cite{somepalli2023understanding}, tuning text embeddings~\cite{wen2024detecting}, initial latents~\cite{jeon2025understanding}, or directly adjusting guidance scale~\cite{jain2025classifier}. 
However, these approaches generally face two major limitations: 
\begin{itemize}
    \item \textbf{Dependence on re-generation:} 
    Most existing inference-time methods operate on a \textit{detect-then-retry} basis. They cannot correct the trajectory seamlessly. Instead, detection triggers an abort, requiring the generation to be restarted with modified inputs, which increases latency and computational cost.
    \item \textbf{Heuristic or static interventions:} Mitigation typically relies on fixed rules or static adjustments to prompts, embeddings, or guidance scales, which compromise semantic fidelity and fail to adapt to the dynamic nature of memorization, leading to over- or under-mitigation.
\end{itemize}

To overcome these limitations, we propose a principled framework for \textbf{on-the-fly detection and adaptive mitigation} of memorization in diffusion models, based on the modeling of \textit{empirical stability regions} of the generation dynamics. 
Our method is inspired by a key observation: when generating memorized samples, the PNDM (Pseudo-Numerical methods for Diffusion Models) scheduler~\cite{liu2022pseudonumericalmethodsdiffusion}, i.e., the default scheduler of Stable Diffusion v1~\cite{rombach2022highresolutionimagesynthesislatent}, often produces noticeable visual degradation, such as a \textit{``broken''} appearance, while non-memorized prompts yield normal generations (see Fig. \ref{fig:intro}). 
We trace this phenomenon to a fundamental cause: \textbf{memorization manifests as numerical instability in the latent dynamics}. 
Specifically, we observe that memorization causes anomalous deviations in the latent update per step ($\delta_t$) from model's \textit{empirical stable region}, which can be estimated from the model's normal behavior. 
We further prove that for non-memorized prompts, $\delta_t$ remains within this stability region, while for memorized prompts, it exhibits large, erratic fluctuations that lead to unstable denoising trajectories and visible artifacts in generated images. 
Notably, such instability is not limited to PNDM: with non-numerical schedulers like DDIM~\cite{song2022denoisingdiffusionimplicitmodels}, even when the generated images appear visually plausible, we still observe similar anomalies throughout the generation in the internal latent space. 

Building on this insight, we propose a \textbf{high-precision, step-wise} detection system that monitors the deviation of $\delta_t$ from the empirical stability region at each denoising step.
This enables accurate diagnosis of whether the current generation trajectory is tending toward memorization. 
Detected instabilities trigger an \textbf{adaptive mitigation} strategy that distinguishes between mild and strong memorization: mild cases are regularized early through latent-magnitude constraints, while strong cases require stricter joint constraints on both latent and update norms throughout the trajectory. 
Both mitigations are performed based on the estimated stability regions from a small set of normal prompts (e.g., $50$ generations from LAION-400M~\cite{schuhmann2022laion}). 
The entire process operates \textbf{on the fly}, with no modification to the prompt or guidance, ensuring both semantic fidelity and image quality.
Table~\ref{tab:baseline_comparison} summarizes representative baseline methods, highlighting that our method uniquely achieves all three capabilities.

Our main contributions are summarized as follows:
\begin{itemize}
    \item We present the first systematic analysis linking degraded memorization generation to numerical instability in diffusion sampling, revealing distinctive anomalies in latent updates ($\delta_t$) magnitudes. 
    \item We introduce the concept of \textit{empirical stability regions}, grounded in numerical stability theory and extended to broader scenarios, to characterize and diagnose normal versus memorized generation in diffusion models.
    \item We develop an on-the-fly, diagnosis-driven framework that adaptively constrains the generation trajectories only when instability is detected, surpassing existing heuristic or static methods. 
    \item Our approach achieves superior detection accuracy (AUC $> 0.999$) in identifying memorization and completely eliminates memorized generations ($0.0\%$ rate) while preserving high image quality and semantic fidelity, with minimal overhead ($\approx0.01$\,s per image).
\end{itemize}

\begin{table}[t]
\centering
\caption{Comparison with representative baselines. Our method uniquely achieves both on-the-fly and adaptive mitigation with integrated detection. }
\resizebox{\linewidth}{!}{
\begin{tabular}{lcccccc}
\toprule
{Method}  & RTA \cite{somepalli2023understanding}           & Wen et al. \cite{wen2024detecting}      & Ren et al.  \cite{ren2024unveiling}     &  Jeon et al. \cite{jeon2025understanding} & Jain et al. \cite{jain2025classifier}   & Ours             \\ \midrule
With detection?             & \XSolidBrush & \CheckmarkBold & \CheckmarkBold &  \CheckmarkBold 
&\XSolidBrush & \CheckmarkBold \\ 
On-the-fly mitigation?       & \XSolidBrush & \XSolidBrush   & \XSolidBrush   &  \XSolidBrush   
&\XSolidBrush & \CheckmarkBold \\ 
Adaptive mitigation?       & \XSolidBrush & \XSolidBrush   & \XSolidBrush   &  \XSolidBrush   &\CheckmarkBold & \CheckmarkBold \\
\bottomrule  
\end{tabular}
}
\label{tab:baseline_comparison}
\end{table}
\section{Related Work}

\subsection{Memorization Detection in Diffusion Models}
Recent studies have revealed that diffusion models are prone to memorizing and reproducing training images~\cite{chavhan2024memorizedimagesdiffusionmodels,bonnaire2025diffusionmodelsdontmemorize,luo2024privacypreservinglowrankadaptationmembership,guan2025ufidunifiedframeworkinputlevel}. 
Early detection methods rely on post-hoc comparisons between generated samples and the training set, such as using Euclidean distance~\cite{carlini2023extracting} or features extracted by self-supervised copy detection (SSCD) models~\cite{pizzi2022self}. 
However, these image-level approaches are often confounded by complex patterns and require the generation to be fully completed.

To enable earlier detection, subsequent works have explored internal model signals. 
Wen et al.~\cite{wen2024detecting} and Ren et al.~\cite{ren2024unveiling} demonstrated that memorized prompts induce significantly larger guidance magnitudes and distinct attention patterns compared to non-memorized ones. 
Beyond signal magnitude, recent studies have emphasized the \textit{dynamics} and \textit{geometry} of the generative process. 
Brokman et al.~\cite{brokman2025tracking} analyzed the memorization geometry via the mean-curvature of the CFG vector field, while Jeon et al.~\cite{jeon2025understanding} linked memorization to the sharpness of the initial latent probability landscape. 
These works provide a deeper theoretical understanding of why memorization manifests as anomalous dynamics.
Other approaches focus on specific scenarios, such as detecting memorized generations with particular images~\cite{jiang2025imagelevel} or identifying local memorization patterns~\cite{chen2025exploringlocalmemorizationdiffusion}. 

Despite these advances, most detection methods generally serve as diagnostic tools rather than real-time controllers. 
Even when used for mitigation, they typically trigger a re-generation rather than enabling seamless, on-the-fly intervention during the generation trajectory.

\subsection{Memorization Mitigation in Diffusion Models}
Mitigation strategies can be broadly categorized into training-time and test-time interventions. 
Training-time methods generally introduce regularization or data filtering to reduce memorization risk. 
Common approaches include monitoring guidance magnitude~\cite{wen2024detecting} or attention patterns~\cite{ren2024unveiling} during training, and removing batches that exceed predefined thresholds. 
Other methods augment prompts to reduce replication likelihood~\cite{somepalli2023understanding}. 
While effective, these techniques require model retraining or fine-tuning, which is computationally prohibitive for large-scale models and inflexible for deployed systems.

Test-time approaches, on the other hand, aim to mitigate memorization during the inference phase. 
Typical methods involve modifying prompts~\cite{somepalli2023understanding,chen2025enhancing}, optimizing text embeddings~\cite{wen2024detecting,chen2025exploringlocalmemorizationdiffusion}, optimizing initial latents~\cite{jeon2025understanding}, adjusting attention weights~\cite{ren2024unveiling}, or tuning guidance scales~\cite{jain2025classifier}. 
However, these strategies face two critical limitations. 
First, they often operate on a \textit{detect-then-retry} paradigm: mitigation is applied only after a problematic generation is detected partially or fully, triggering a restart with modified inputs. This dependency on regeneration significantly increases inference latency.
Second, interventions are typically applied as heuristic or static adjustments, e.g., fixed guidance reduction or embedding optimization at initial step, rather than adapting dynamically to the evolving generation trajectory. This can lead to either under-mitigation or over-constraint, potentially harming image quality.

In contrast, our work introduces an adaptive, trajectory-aware framework that enables on-the-fly monitoring and seamless intervention without the need for regeneration, effectively ensuring both safety and efficiency.

\section{Preliminaries} \label{sec:pilot}

\subsection{Diffusion Models}
Diffusion models~\cite{ho2020denoising,song2020score} define a forward noising process that gradually perturbs a data sample \( \mathbf{x}_0 \sim p_{\text{data}}(\mathbf{x}_0) \) into pure noise through a continuous-time stochastic differential equation (SDE):
\begin{equation}
    d\mathbf{x}(t) = f(\mathbf{x}(t), t)\,dt + g(t)\,d\mathbf{w},
\end{equation}
where \( \mathbf{w} \) is the standard Wiener process, \( f(\mathbf{x}(t), t) \) and \( g(t) \) are the drift and diffusion coefficients, respectively. 

The sampling is managed by the reverse-time SDE, which is formulated as:
\begin{equation}
    d\mathbf{x}(t) = \left[ f(\mathbf{x}(t), t) - g(t)^2 \nabla_{\mathbf{x}} \log p_t(\mathbf{x}(t)) \right] dt + g(t)\,d\bar{\mathbf{w}},
\end{equation}
where \( \bar{\mathbf{w}} \) denotes the reverse-time Wiener process, and \( \nabla_{\mathbf{x}} \log p_t(\mathbf{x}(t)) \) is approximated by a neural network \( \epsilon_\theta(\mathbf{x}_t, t) \) trained to predict the added noise.

This SDE formulation unifies the forward and reverse processes, enabling sampling through either stochastic or deterministic solvers.

\subsection{Text-Conditional Latent Diffusion Models}
In this work, we focus on latent diffusion models (LDMs)~\cite{rombach2022highresolutionimagesynthesislatent} for text-to-image generation, e.g., Stable Diffusion, where the generative process is performed in a latent space rather than the original pixel space.
During sampling, a pretrained autoencoder \( \mathcal{E} \) maps an image \( \mathbf{x} \in \mathbb{R}^{H \times W \times 3} \) to a latent representation \( \mathbf{z} = \mathcal{E}(\mathbf{x}) \in \mathbb{R}^{h \times w \times d} \).
The diffusion model is trained to denoise samples with prediction \( \epsilon_\theta(\mathbf{x}_t, t, c) \) in this latent space given a text prompt \( c \), and a decoder \( \mathcal{D} \) is used to reconstruct the image \( \mathbf{x} = \mathcal{D}(\mathbf{z}) \)~\cite{Kingma_2019}.

\subsection{Schedulers}
Schedulers play a crucial role in diffusion models by controlling the denoising trajectory during the sampling process, determining how the model transitions from pure noise to a clean sample.

\noindent
\textbf{Denoising diffusion implicit models (DDIM)}~\cite{song2022denoisingdiffusionimplicitmodels} introduce a non-stochastic, deterministic alternative to sampling. The DDIM update at step \( t-1 \) is given by:
\begin{equation}
    \mathbf{z}_{t-1} = \sqrt{\bar{\alpha}_{t-1}} \hat{\mathbf{z}}_0^{(t)} + \sqrt{1 - \bar{\alpha}_{t-1}} \boldsymbol{\epsilon}_\theta(\mathbf{z}_t, t, c),
\end{equation}
where \( \hat{\mathbf{z}}_0^{(t)} = \frac{1}{\sqrt{\bar{\alpha}_t}} \left( \mathbf{z}_t - \sqrt{1 - \bar{\alpha}_t} \boldsymbol{\epsilon}_\theta(\mathbf{z}_t, t, c) \right) \) denotes the model's estimate of the original clean sample at the current step.
This formulation corresponds to solving an ODE defined by the probability flow of the diffusion process.

\noindent
\textbf{Pseudo numerical methods for diffusion models (PNDM)}~\cite{liu2022pseudonumericalmethodsdiffusion} accelerate DDIM by using high-order numerical solvers to integrate the underlying ODE.
Given the DDIM update as an Euler discretization, PNDM improves sampling efficiency and quality via higher-order methods such as Runge–Kutta~\cite{BUTCHER1996247} or Adams-Bashforth schemes~\cite{1991thirdorder}.
The fourth-order Adams-Bashforth (AB4) method used by PNDM is defined as:
\begin{equation}  \label{eq:pndm}
\mathbf{z}_{t+1} = \mathbf{z}_t + \Delta t \cdot \sum_{i=0}^{3} b_i f_\theta(\mathbf{z}_{t-i}, t-i, c),
\end{equation}
where \( b_i \) are fixed coefficients derived from the method, \( \Delta t \) is the integration step size, and \( f_\theta(\mathbf{z}, t, c) = -\epsilon_\theta(\mathbf{z}, t, c) \) represents the model's denoising prediction.\footnote{Following the convention of numerical ODE solvers, the equation is written in the forward direction (\( t \to t+1 \)), while the actual sampling process proceeds in the reverse direction (\( t \to t-1 \)), i.e., from noise to data. }
The formulation bridges the gap between the deep generative model and classical ODE solvers by viewing the denoising trajectory as an integration path over estimated derivatives.

\section{Understanding Degraded Generations for Memorized Prompts}

\begin{figure*}[t]
\centering
\includegraphics[width=\linewidth]{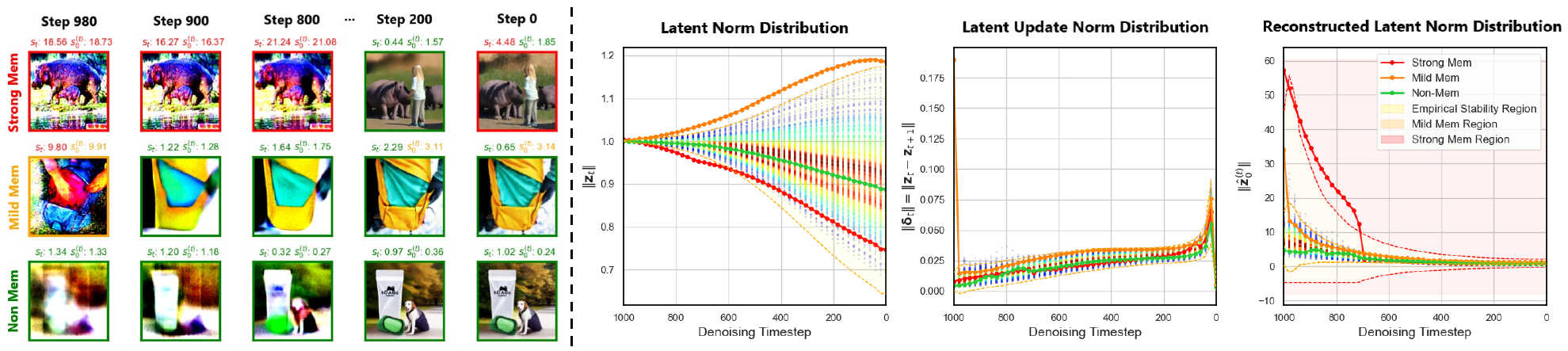}  
\caption{
    \textbf{Left:} On-the-fly detection and mitigation progress. Each row visualizes predicted $\mathbf{x}_0^{(t)} = \mathcal{D}(\hat{\mathbf{z}}_0^{(t)})$ across denoising steps for strong, mild, and non-memorized prompts, respectively. 
    Red, orange, and green borders denote strong, mild, and no mitigation applied. 
    Strong memorization requires continuous mitigation, mild memorization only a few early interventions, and non-memorized prompts evolve normally without intervention.
    \textbf{Right:} Empirical stability regions (yellow-shaded) derived from normal generations, defined on latent norm ($\|\mathbf{z}_t\|$), latent update norm ($\|\delta_t\| = \|\mathbf{z}_t - \mathbf{z}_{t+1}\|$), and reconstructed latent norm ($\|\hat{\mathbf{z}}_0^{(t)}\|$), respectively.
    Colored trajectories represent generation progress under our mitigation (red/orange/green for strong/mild/non-memorization). 
    The trajectories of strong and mild prompts are adaptively constrained within their stability regions, while non-memorized prompts remain stable throughout generation.
}
\label{fig:method}
\end{figure*}

\subsection{Empirical Stability Region via Latent Trajectories}
As discussed in Sec.~\ref{sec:pilot}, PNDM employs high-order numerical solvers (e.g., AB4) to integrate the latent ODE. 
From the perspective of classical numerical analysis, the stability of such solvers is strictly dictated by the local linearization of the dynamics $d\mathbf{z}/dt = \mathbf{J}\mathbf{z}$, represented by the Jacobian matrix $\mathbf{J}$. 
Formally, stability requires that the \emph{scaled eigenvalues} $\nu := \lambda \Delta t$ (where $\lambda$ are eigenvalues of $\mathbf{J}$) lie within the solver's absolute stability region $\mathcal{R}_{\mathrm{AB}} \subset \mathbb{C}$~\cite{hairer1993solving,smith1985numerical}.
If this condition is violated ($\nu \notin \mathcal{R}_{\mathrm{AB}}$), the gradient updates accumulated by the solver will inherently cause the numerical solution to oscillate or diverge.

While this theoretical insight provides an underlying explanation for the ``broken'' artifacts observed in memorized generations, applying it directly to diffusion models is non-trivial.
In diffusion sampling, the latent $\mathbf{z}_t$ evolves iteratively governed by the denoising network $f_\theta(\mathbf{z}_t, t, c)$, which is parameterized by a complex, deep neural network (U-Net).
This makes the explicit computation of its local Jacobian spectrum computationally intractable. 
Consequently, we cannot analytically seek the stability boundaries based on eigenvalue analysis.
However, the manifestation of instability remains consistent with the classical theory, i.e., \textit{anomalously large or erratic updates in the solution trajectory.}
Therefore, instead of analyzing the inaccessible eigenvalues, we propose an \emph{empirical stability region} that approximates the stable behavior of the denoiser empirically by directly monitoring the magnitude of latent updates distributions.

Let $\delta_t := \mathbf{z}_{t} - \mathbf{z}_{t+1}$ denote the latent update at step $t$, corresponding to a numerical integration step in the reverse-time diffusion trajectory. Given a set of normal (i.e., non-memorized) prompts $c$ and their generation trajectories, we collect latent updates $\delta_t$ and compute their statistical distribution over timesteps. We define the empirical stability region as follows:

\begin{definition}[Empirical Stability Region of $\delta_t$] \label{def:stability-region}
Let $\mu_t$ and $\sigma_t$ denote the mean and standard deviation of $\|\delta_t\|$ estimated from non-memorized samples at timestep $t$. The empirical stability region at step $t$ is defined as the interval:
\[
\mathcal{R}_{\delta}^{(t)} := \left[ \mu_t - \gamma \sigma_t,\ \mu_t + \gamma \sigma_t \right],
\]
where $\gamma$ is a tunable confidence width (e.g., $\gamma=3$).
\end{definition}
This region captures the typical scale of latent transitions during generation, as shown in Fig. \ref{fig:method} (right). Deviations beyond this range indicate abnormally large steps, which may reflect instability.

\begin{theorem}[Normal Trajectories Stability]
\label{thm:normal-stability}
Let $\delta_t$ denote the latent updates of normal prompts, and let $\mathcal{R}_{\delta}^{(t)}$ be the empirical stability region defined in Def. \ref{def:stability-region}. Then with high probability, we have:
\[
\|\delta_t\| \in \mathcal{R}_{\delta}^{(t)}, \quad \forall t \in \{1, \dots, T\}.
\]
In other words, latent trajectories from non-memorized prompts tend to stay within their corresponding empirical stability region throughout the generation.
\end{theorem}

\begin{proof}[Sketch]
The empirical region $\mathcal{R}_{\delta}^{(t)}$ is constructed from the empirical mean $\mu_t$ and std. $\sigma_t$ over $\|\delta_t\|$ sampled from the reference set. Under the assumption that these updates follow a sub-Gaussian distribution~\cite{ho2020denoising,song2020score}, concentration bounds ensure that most values lie within $[\mu_t \pm \gamma \sigma_t]$. Thus, $\mathcal{R}_{\delta}^{(t)}$ captures the high-probability support of normal trajectories.
Detailed proofs can be found in Appendix. \ref{sec:proof}. 
\end{proof}

\subsection{Instability from Over-Guidance in Memorized Generations}
To understand why memorized prompts lead to instability, we examine the mechanism of classifier-free guidance (CFG)~\cite{ho2022classifierfreediffusionguidance}. Recall that the CFG prediction is computed as:
\begin{equation}  \label{eq:cfg}
\epsilon_\theta^{\text{CFG}}(\mathbf{z}_t, t, c) = \epsilon_\theta(\mathbf{z}_t, t, \phi) + w \cdot \left[ \epsilon_\theta(\mathbf{z}_t, t, c) - \epsilon_\theta(\mathbf{z}_t, t, \phi) \right],
\end{equation}
where $w$ is the guidance scale, $c$ is the conditional prompt, and $\phi$ is the unconditional generation. 
Since memorization corresponds to the model overfitting to specific training samples~\cite{bonnaire2025diffusion}, prior work has shown that the CFG often exhibits sharp gradients and abnormally large magnitudes for memorized prompts~\cite{wen2024detecting}. 

In the context of multi-step solvers like PNDM, the latent update step $\delta_t$ is directly proportional to this prediction:
\begin{equation}
\delta_t = \mathbf{z}_t - \mathbf{z}_{t+1} \propto \Delta t \cdot \epsilon_\theta^{\text{CFG}}(\mathbf{z}_t, t, c).
\end{equation}
Consequently, the guidance scale $w$ amplifies the already excessive gradients of memorized prompts, resulting in anomalously large latent updates $\delta_t$. We formalize this intuition in the following proposition.

\begin{proposition}[Instability of Memorized Trajectories]
\label{prop:mem-instability}
Let $c_{mem}$ be a memorized prompt where the conditional noise prediction norm significantly dominates the unconditional one, such that $\|\epsilon_\theta(\mathbf{z}_t, t, c_{mem})\| \ge \beta \|\epsilon_\theta(\mathbf{z}_t, t, \phi)\|$ for a large factor $\beta > 1$.
With guidance scale $w$, the magnitude of the latent update $\|\delta_t\|$ for the memorized prompt satisfies:
\[
\|\delta_t\|_{mem} \approx w \cdot C \cdot \|\epsilon_\theta(\mathbf{z}_t, t, c_{mem})\|,
\]
where $C$ is a constant related to the step size. Consequently, due to the excessively large gradient norm $\|\epsilon_\theta(\mathbf{z}_t, t, c_{mem})\|$ (large $\beta$), the resulting update $\|\delta_t\|_{mem}$ exceeds the upper bound of the stability region $\mathcal{R}_{\delta}^{(t)} = [\mu_t \pm \gamma \sigma_t]$ calibrated on normal samples, causing the trajectory to diverge from the stable manifold.
\end{proposition}

\begin{proof}[Sketch]
For memorization, the model output is highly peaked around the memorized data point, implying $\|\epsilon_\theta(\mathbf{z}_t, t, c_{mem})\| \gg \|\epsilon_\theta(\mathbf{z}_t, t, \phi)\|$. Substituting this into Eq.~\ref{eq:cfg} and neglecting the smaller unconditional term yields $\|\epsilon_\theta^{\text{CFG}}\| \approx w \|\epsilon_\theta(\mathbf{z}_t, t, c_{mem})\|$. Since $\delta_t$ is computed from $\epsilon_\theta^{\text{CFG}}$ scaled by the step size (constant $C$), we obtain $\|\delta_t\|_{mem} \approx w \cdot C \cdot \|\epsilon_\theta(\mathbf{z}_t, t, c_{mem})\|$. Because the stability threshold is derived from normal samples with much smaller gradients, this amplified update naturally violates the bound. The instability is thus driven by the intrinsic sharpness of the memorized minima.
Detailed proofs can be found in Appendix. \ref{sec:proof}. 
\end{proof}

Empirically, this instability is most pronounced during the early-to-mid generation steps when the coarse image structure is formed. The excessive updates $\delta_t$ cause the trajectory to ``overshoot'', manifesting in the final image as high-frequency artifacts, such as over-sharpened edges or broken textures (see Fig.~\ref{fig:intro}).

This effect is further exacerbated by high-order solvers like PNDM (AB4). Unlike single-step methods (e.g., Euler), AB4 accumulates momentum from previous steps. Consequently, an initial deviation caused by over-guidance is not immediately corrected but rather propagated and amplified over subsequent steps, leading to severe visual degradation.

\subsection{Extending Stability Analysis to Reconstructed Initial Latent}
Having attributed the visual degradation in PNDM to numerical instability, we now broaden our scope to the widely-used DDIM sampler~\cite{song2022denoisingdiffusionimplicitmodels}. We demonstrate that \textit{deviating from the stability region} is a universal signature of memorization: even when DDIM produces visually clean images, its underlying latent trajectory still significantly diverges from the stable manifold.

To examine whether instability persists, we turn to the reconstruction mechanism of DDIM. At each timestep $t$, DDIM performs deterministic recovery of an estimate of the initial latent $\hat{\mathbf{z}}_0$ via:
\[
\hat{\mathbf{z}}_0^{(t)} = \frac{1}{\sqrt{\bar{\alpha}_t}} \left( \mathbf{z}_t - \sqrt{1 - \bar{\alpha}_t} \, \epsilon_\theta(\mathbf{z}_t, t, c) \right).
\]
This quantity is meant to approximate the actual $z_0$ at current timestep $t$, i.e., the latent that would be decoded by the autoencoder into the final image $x_0$. As such, it acts as an intermediate ``anchor'' that encapsulates the accumulated influence of all denoising steps up to time $t$.
Similar to our analysis of $\delta_t$, we define an empirical stability region for $\hat{\mathbf{z}}_0^{(t)}$. 

\begin{definition}[Empirical Stability Region of $\hat{\mathbf{z}}_0$] \label{def:z0-stability-region}
Let $\mu_0^{(t)}$ and $\sigma_0^{(t)}$ denote the mean and standard deviation of $\|\hat{\mathbf{z}}_0^{(t)}\|$ estimated from normal prompts. Then the empirical stable region is defined as:
\[
\mathcal{R}_0^{(t)} := \left[ \mu_0^{(t)} - \gamma \sigma_0^{(t)},\ \mu_0^{(t)} + \gamma \sigma_0^{(t)} \right],
\]
with the same confidence parameter $\gamma$ as before.
\end{definition}

This allows us to assess whether the reconstructed initial latent remains within a stable range. 
Specifically, we show that for normal prompts, the reconstructed $\hat{\mathbf{z}}_0^{(t)}$ remains within this empirical stability region across all timesteps.

\begin{theorem}[Normal $\hat{\mathbf{z}}_0$ Stability]
\label{thm:z0-stability}
Let $\hat{\mathbf{z}}_0^{(t)}$ be the reconstructed initial latent at step $t$ in DDIM, and let $\mathcal{R}_0^{(t)}$ be the stability region in Def. \ref{def:z0-stability-region}. Then, for normal generations, we have:
\[
\hat{\mathbf{z}}_0^{(t)} \in \mathcal{R}_0^{(t)}, \quad \forall t \in \{1, \dots, T\}.
\]
\end{theorem}

\begin{proof}[Sketch]
The DDIM formulation deterministically maps $\mathbf{z}_t$ and $\epsilon_\theta$ to $\hat{\mathbf{z}}_0^{(t)}$. For normal prompts, $\epsilon_\theta$ behaves regularly, leading to a stable and bounded reconstruction $\hat{\mathbf{z}}_0^{(t)}$. 
Detailed proofs can be found in Appendix. \ref{sec:proof}. 
\end{proof}

Empirically, we observe that $\hat{\mathbf{z}}_0^{(t)}$ for memorized prompts often lies distinctively outside the stability region $\mathcal{R}_0^{(t)}$ at intermediate timesteps, even when the final output appears visually indistinguishable from a normal natural image (see Fig.~\ref{fig:intro}). 
This confirms that the anomaly lies in the \textit{sampling dynamics} rather than the final state.
This indicates that DDIM can conceal underlying instability in its latent dynamics, revealing that memorization effects persist internally as trajectory deviations.

Crucially, this analysis of $\hat{\mathbf{z}}_0$ confirms that memorized trajectories in DDIM are effectively ``broken'' internally, driven by the similar excessive guidance gradients as in PNDM. 
Consequently, this internal instability inevitably forces the latent updates $\delta_t$ to deviate from their stability region $\mathcal{R}_{\delta}^{(t)}$ as well.
Thus, the unified concept of a latent stability region is sampler-agnostic: while $\hat{\mathbf{z}}_0$ serves as an analytical tool to uncover hidden instability, the latent update $\delta_t$ remains a universal and sufficient signature for identifying memorizations across both samplers.

\begin{table*}[t]
\centering
\caption{Memorization detection results with AUC, TPR@$1\%$FPR on SD 1.4. The best results are in \textbf{bold}.}
\label{tab:detect_main}
\begin{tabular}{lccccccc}
\toprule
\multicolumn{1}{c}{\multirow{2}{*}{Method}} & \multirow{2}{*}{Num of Seeds} & \multicolumn{2}{c}{First 3 Steps}      & \multicolumn{2}{c}{All Steps}          & \multicolumn{2}{c}{Avg. on Steps}          \\  \cline{3-8} 
\multicolumn{1}{c}{}                        &                               & AUC↑               & TPR@1\%FPR↑       & AUC↑               & TPR@1\%FPR↑       & AUC↑               & TPR@1\%FPR↑       \\ \midrule
\multirow{3}{*}{Ren et al.~\cite{ren2024unveiling}}                 & 1                             & \multicolumn{2}{c}{\multirow{3}{*}{-}} & 0.9660             & 0.6220            & \multicolumn{2}{c}{\multirow{3}{*}{-}} \\
                                            & 4                             & \multicolumn{2}{c}{}                   & 0.9383             & 0.6940            & \multicolumn{2}{c}{}                   \\
                                            & 8                             & \multicolumn{2}{c}{}                   & 0.9395             & 0.6980            & \multicolumn{2}{c}{}                   \\ \midrule
\multirow{3}{*}{Wen et al.~\cite{wen2024detecting}}                 & 1                             & \textbf{0.9880}             & 0.9160            & 0.9831             & \textbf{0.9140}   & 0.9859             & \textbf{0.9253}   \\
                                            & 4                             & 0.9955             & 0.9700            & 0.9959             & \textbf{0.9820}   & 0.9964             & \textbf{0.9775}   \\
                                            & 8                             & {0.9985}    & 0.9920            & 0.9977             & 0.9880            & 0.9982             & 0.9850            \\ \midrule
\multirow{3}{*}{Jeon et al.~\cite{jeon2025understanding}}                & 1                             & 0.9779             & 0.8480            & \multicolumn{2}{c}{\multirow{3}{*}{-}} & \multicolumn{2}{c}{\multirow{3}{*}{-}} \\
                                            & 4                             & \textbf{0.9977}             & \textbf{0.9820}   & \multicolumn{2}{c}{}                   & \multicolumn{2}{c}{}                   \\
                                            & 8                             & \textbf{0.9994}        & \textbf{0.9940}   & \multicolumn{2}{c}{}                   & \multicolumn{2}{c}{}                   \\ \midrule
\multirow{3}{*}{Ours}                       & 1                             & 0.9878             & \textbf{0.9220}   & \textbf{0.9858}    & 0.9040            & \textbf{0.9868}    & 0.9042            \\
                                            & 4                             & {0.9964}    & 0.9720            & \textbf{0.9983}    & 0.9800            & \textbf{0.9975}    & 0.9764            \\
                                            & 8                             & 0.9982             & 0.9900            & \textbf{0.9995}    & \textbf{0.9920}   & \textbf{0.9992}    & \textbf{0.9874} \\ \bottomrule
\end{tabular}
\end{table*}

\section{Detecting and Mitigating Memorization via Stability Regions}

\subsection{Detection via Instable Latent Updates}
Based on the empirical stability region in Def. \ref{def:stability-region}, we can detect memorization artifacts by monitoring deviations from the region during generation.
At each denoising step $t$, we compute the absolute Z-score of the latent update norm:
\[
s_t = \left|\frac{\|\delta_t\| - \mu_t}{\sigma_t}\right|.
\]
A high $s_t$ indicates an anomalous update, suggesting potential memorization. We summarize a sample's trajectory by the maximum absolute Z-score over the first $s$ steps\footnote{Note that taking the maximum Z-score across $s$ steps introduces a multiple testing problem, which results in an expected $\sqrt{\log s}$ shift in $S_\text{mem}$ in our multi-step detection scenario. However, since this systemetic shift applies to both memorized and normal samples, it does not affect the relative detection performance.}:
\[
S_{\text{mem}} = \max_{1 \leq t \leq s} s_t.
\]
This approach efficiently highlights sharp deviations from normal behavior and is applicable to any diffusion sampler. Notably, high-order solvers like PNDM further amplify such anomalies, enhancing detection sensitivity. 

Crucially, rather than serving merely as a post-hoc scoring tool for filtering memorized generations, this dynamically calibrated step-wise metric $s_t$ is intrinsically designed to act as a real-time monitor. The true value of this diagnostic signal lies in its ability to pinpoint the exact moment of instability, which serves as the fundamental prerequisite for our on-the-fly mitigation strategy.

\subsection{Mitigation via On-the-Fly Stability-Constrained Adaptive Sampling}

While monitoring $s_t$ effectively flags potential memorization, tailoring the mitigation intensity requires gauging the severity of the instability.
To this end, we incorporate the reconstructed initial latent $\hat{\mathbf{z}}_0$ as a semantic stability proxy, allowing us to differentiate between mild and strong memorization artifacts.

We thus propose a two-level adaptive mitigation strategy. 
To facilitate this, we define an additional empirical stability region $\mathcal{R}_z^{(t)}$ for the latent norm $\|\mathbf{z}_t\|$, analogous to $\mathcal{R}_{\delta}^{(t)}$ in Def. \ref{def:stability-region}.
Let $t^*$ be the first timestep (within the initial $k$ steps) where both the latent update Z-score $s_t$ and the reconstructed initial latent Z-score $s_0^{(t)}$ surpass a mild threshold $\tau_{\text{mild}}$. 
Based on the magnitude of $s_0^{(t^*)}$, we apply the following specific interventions:

\noindent
\textbf{1. Mild Memorization ($\tau_{\text{mild}} < s_0^{(t^*)} \leq \tau_{\text{strong}}$):}
\begin{itemize}
    \item For the first $k$ steps, if both $s_t$ and $s_0^{(t)}$ exceed $\tau_{\text{mild}}$, and $s_0^{(t)} \leq \tau_{\text{strong}}$, we rescale $\hat{\mathbf{z}}_0^{(t)}$ to the empirical mean $\mu_0^{(t)}$.
    \item Throughout the trajectory, we constrain $\|\mathbf{z}_t\|$ to remain within $\mathcal{R}_z^{(t)}$ only when $\|\delta_t\|$ exceeds $\mathcal{R}_{\delta}^{(t)}$.
\end{itemize}

\noindent
\textbf{2. Strong Memorization ($s_0^{(t^*)} > \tau_{\text{strong}}$):}
\begin{itemize}
    \item At every step, if $s_0^{(t)} > \tau_{\text{mild}}$, we rescale $\hat{\mathbf{z}}_0^{(t)}$ to $\mu_0^{(t)}$.
    \item Both $\|\mathbf{z}_t\|$ and $\|\delta_t\|$ are enforced to remain within their respective empirical stability regions.
\end{itemize}

This on-the-fly, trajectory-aware mitigation ensures that strong memorization is robustly suppressed, while mild cases are gently regularized to preserve diversity and semantics.
The mitigation process is visualized in Fig. \ref{fig:method}, where we show how the generation is constrained to remain within the empirical stability regions.
Complete algorithms are in Appendix \ref{sec:algorithm}.

\begin{table*}[t]
\centering
\caption{Memorization detection results with AUC, TPR@$1\%$FPR on SD 1.4 (num of seeds $=8$), using PNDM and DDIM. }
\label{tab:detection_results}
\begin{tabular}{lccccccc}
\toprule
\multicolumn{1}{c}{}                                 &                             & \multicolumn{2}{c}{First 3 Steps}                                                 & \multicolumn{2}{c}{All Steps}                                            & \multicolumn{2}{c}{Avg. on Steps}                                                     \\
\multicolumn{1}{c}{\multirow{-2}{*}{Method}}         & \multirow{-2}{*}{Scheduler} & AUC↑                                    & TPR@1\%FPR↑                             & AUC↑                           & TPR@1\%FPR↑                             & AUC↑                                    & TPR@1\%FPR↑                             \\ \midrule
                                                     & {DDIM} & \multicolumn{2}{c}{{}}                                       &\textbf{0.9444}  & \textbf{0.7240}           & \multicolumn{2}{c}{{}}                                       \\ 
\multirow{-2}{*}{Ren et al.~\cite{ren2024unveiling}}                         & {PNDM} & \multicolumn{2}{c}{\multirow{-2}{*}{{-}}}                    & {-0.0049} & {-0.0260}          & \multicolumn{2}{c}{\multirow{-2}{*}{{-}}}                    \\ \midrule
{}                              & {DDIM} & {0.9983}           & {0.9860}           & \textbf{0.9978}  & {0.9860}           & {0.9981}           & {0.9825}           \\ 
\multirow{-2}{*}{{Wen et al.~\cite{wen2024detecting}}}  & {PNDM} & {\textbf{+0.0002}} & {\textbf{+0.0060}} & {-0.0001} & {\textbf{+0.0020}} & {\textbf{+0.0001}} & {\textbf{+0.0024}} \\ \midrule
{}                              & {DDIM} & \textbf{0.9994}                                  & 0.9900                                  & \multicolumn{2}{c}{{}}                              & \multicolumn{2}{c}{{}}                                       \\ 
\multirow{-2}{*}{{Jeon et al.~\cite{jeon2025understanding}}} & {PNDM} & {0.0000}          & {\textbf{+0.0040}} & \multicolumn{2}{c}{\multirow{-2}{*}{{-}}}           & \multicolumn{2}{c}{\multirow{-2}{*}{{-}}}                    \\ \midrule
                                                     & {DDIM} & \textbf{0.9982}                         & 0.9900                                  & {0.9997}                & {0.9940}                         & {0.9993}           & 0.9904                                  \\ 
\multirow{-2}{*}{Ours}                               & {PNDM} & {-0.0002}          & {\textbf{+0.0020}} & \textbf{+0.0001}  & \textbf{+0.0020}           & {\textbf{+0.0002}} & {\textbf{+0.0017}}\\ \bottomrule
\end{tabular}
\end{table*}

\subsection{Advantages over Prior Work}

Our framework enables \textbf{on-the-fly trajectory correction}, fundamentally distinguishing it from prior inference-time approaches.
While existing methods~\cite{wen2024detecting,ren2024unveiling} can potentially detect memorization at early steps (e.g., via prediction magnitude), their mitigation paradigms remain distinctively costly. 
Specifically, once memorization is flagged, these methods typically require aborting the current process and applying iterative optimization on prompts or embeddings, followed by a complete \textbf{restart} of the generation. 
Even if detection happens at step $t=1$, the overhead of computing gradients for optimization and restarting the sampling loop is non-negligible.
In contrast, our approach acts as a \emph{streaming filter}: our detection mechanism enables instability monitoring at each step, and our mitigation strategy applies instantaneous projection to pull the latent state back into the stability region as soon as instability is detected, without interrupting the diffusion process.
This allows the generation to proceed seamlessly, enforcing safety constraints with virtually zero additional latency compared to standard inference.
Beyond computational efficiency, our framework offers a theoretical advantage: by grounding detection in numerical stability rather than heuristic magnitude thresholds, it provides a principled explanation for \textit{why} memorization leads to broken generations, unifying prior observations under a consistent dynamical perspective.

\section{Experiments}
\subsection{Memorization Detection}

\noindent
\textbf{Datasets and Models.}
Three datasets are used in this experiment, i.e., one memorized and two non-memorized. 
Following prior works, the memorized set contains $500$ prompts corresponding to known duplicated images in the LAION dataset~\cite{schuhmann2021laion,schuhmann2022laion} constructed by~\cite{webster2023reproducible}, which includes both exact and template verbatim samples. 
For non-memorized sets, the one used for evaluation contains a total of $500$ prompts sampled from the Lexica.art gallery, the COCO-2017 validation set caption, and generated by GPT-4~\cite{achiam2023gpt} following the construction in~\cite{wen2024detecting}.
Another non-memorized set is used for estimating the empirical stability region, which contains $50$ prompts sampled from the LAION-400M~\cite{schuhmann2021laion} dataset.
For detection, we consider pretrained SD 1.4, 1.5, 2.1, with PNDM~\cite{liu2022pseudonumericalmethodsdiffusion} and DDIM~\cite{song2022denoisingdiffusionimplicitmodels}. 
For SD 1.5, 2.1 results, please refer to Appendix~\ref{sec:more_detection_results_apdx}.

\noindent
\textbf{Metrics.}
Following prior works~\cite{wen2024detecting,ren2024unveiling}, we use two standard metrics, Area Under the Receiver Operating Characteristic Curve (AUC) and True Positive Rate at $1\%$ False Positive Rate (TPR@$1\%$FPR).

\noindent
\textbf{Baselines.}
We consider three representative baselines. 
Ren et al.~\cite{ren2024unveiling} uses the entropy of cross-attention maps as a signal, i.e., memorized prompts lead to more stable entropy.
Wen et al.~\cite{wen2024detecting} computes the L2 norm of the guidance $||\epsilon_\theta(z_t,t,c)-\epsilon_\theta(z_t, t, \phi)||_2$.
Jeon et al.~\cite{jeon2025understanding} extends Wen et al.'s approach by incorporating the Hessian of the guidance to account for the local curvature.

\noindent
\textbf{Results.}
As shown in Table~\ref{tab:detect_main}, our method, which monitors the Z-score of the latent update norm, consistently outperforms existing baselines in the vast majority of settings, achieving high AUC and TPR@$1\%$FPR, demonstrating its superior accuracy and reliability.
Such accurate per-step detection performance is crucial for mitigation. 
We also show in Appendix~\ref{sec:ref_prompt_dist_apdx} that our method is insensitive to the distribution and size of the reference set.

\noindent
\textbf{The amplifying effect of PNDM.}
While our method already yields high detection accuracy with standard schedulers like DDIM, PNDM can further enhance the signal contrast.
Due to its multi-step nature, PNDM accumulates the instability inherent to memorized prompts, thereby boosting the performance of methods sensitive to guidance dynamics, as shown in Table~\ref{tab:detection_results}.
In contrast, Ren et al. does not benefit from this, as its attention-entropy-based indicator is relatively independent of the solver's integration trajectory.

\begin{figure}[t]
\centering
\includegraphics[width=0.8\linewidth]{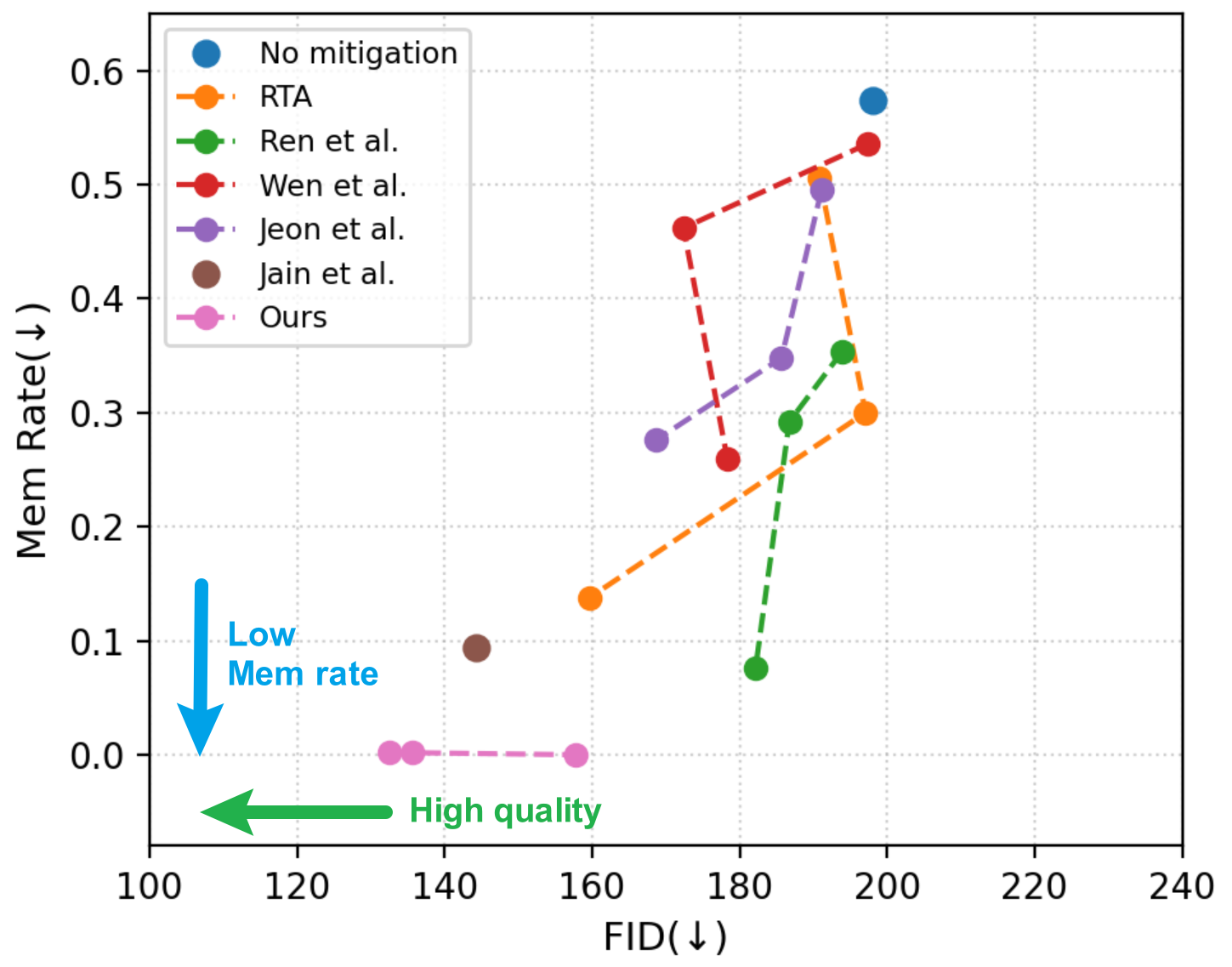}
\caption{Comparison of mitigations on SD 1.4. }
\vspace{-0.1in}
\label{fig:ablation_curve}
\end{figure}

\begin{figure*}[t]
\centering
\includegraphics[width=0.8\linewidth]{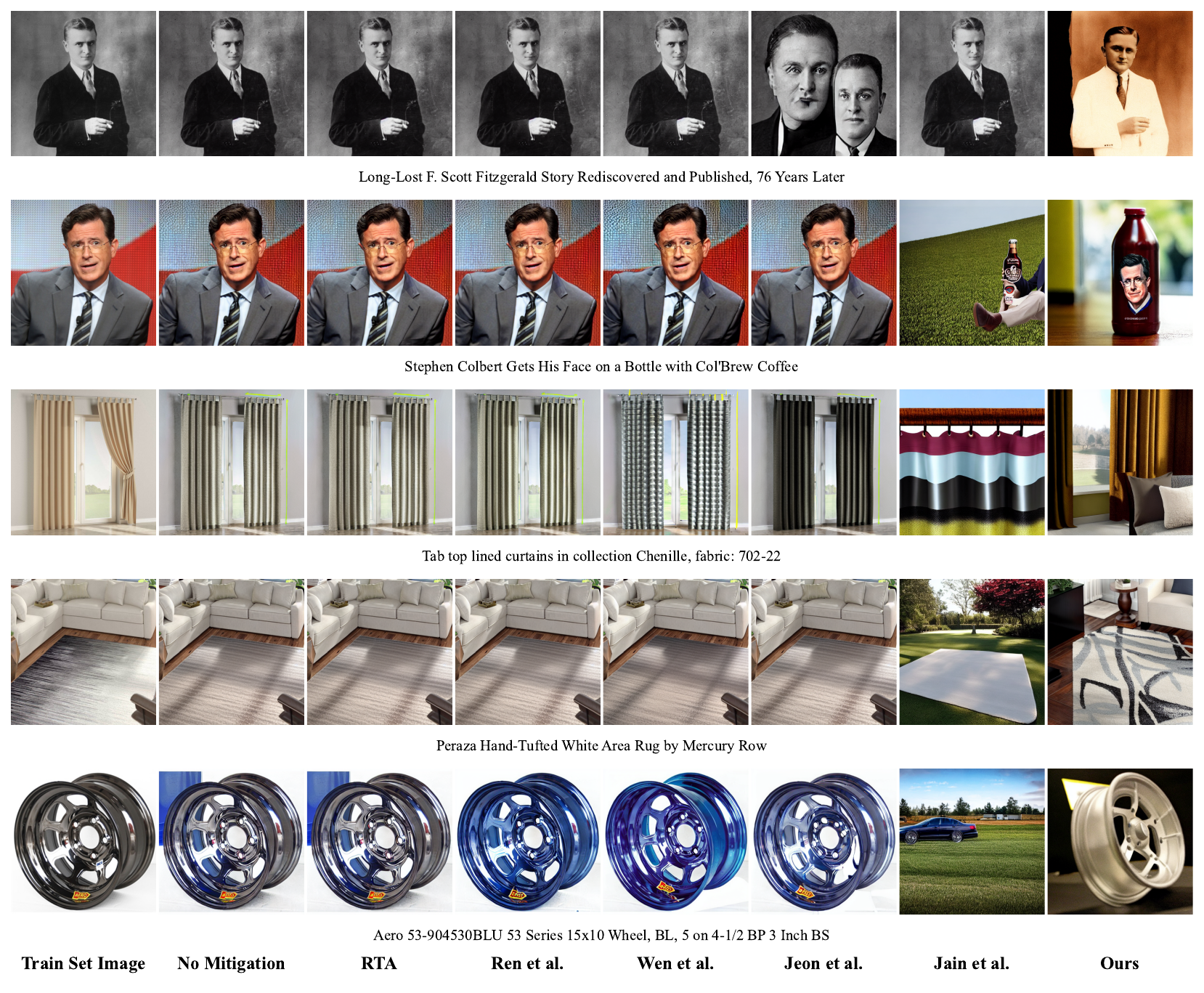}  
\caption{Qualitative results comparing the proposed approach with the baselines on SD 1.4.}
\label{fig:exp}
\end{figure*}

\begin{table}[t]
\centering
\caption{Inference time (s/image) of mitigation.}
\resizebox{\linewidth}{!}{
\begin{tabular}{lccccccc}
\toprule
Method                 & No Mitigation                & RTA                          & Wen et al.                   & Ren et al. & Jeon et al.                  & Jain et al.                  & Ours                         \\ \midrule
Inference Time (s) ↓ & {1.744} & \textbf{1.753} & {2.466} & {2.882} & {50.469} & {3.426} & \underline{1.757} \\ \bottomrule
\end{tabular}
}
\vspace{-0.1in}
\label{tab:time_comparison}
\end{table}

\subsection{Memorization Mitigation}

\noindent
\textbf{Datasets and Models. }
We evaluate mitigation on two scenarios using the DDIM scheduler~\cite{song2022denoisingdiffusionimplicitmodels}: pretrained SD 1.4 (using the same datasets as detection) and SD 1.4 finetuned on a LAION subset to simulate memorization~\cite{wen2024detecting}.
For practical evaluation, all online mitigation thresholds are calibrated to $1\%$ FPR on the LAION-400M reference set.
Due to space constraints, results for the finetuned SD are detailed in Appendix~\ref{sec:more_mitigation_results_apdx}.

\noindent
\textbf{Metrics. }
We use Self-Supervised Copy Detection (SSCD)~\cite{pizzi2022self} to detect whether the generated image is a memorized one, which measures the embedding similarity of two images with an SSL model.
The memorization rate (Mem Rate) evaluates the portion of memorized images generated with SSCD $>0.5$. 
CLIP Score~\cite{hessel2022clipscorereferencefreeevaluationmetric,radford2021learningtransferablevisualmodels} on both memorized and non-memorized sets is used to measure the semantic alignment between the generated image and the prompt.
Fréchet Inception Distance (FID)~\cite{heusel2018ganstrainedtimescaleupdate} is used to evaluate the overall quality and realism of the generated images. 
More details can be found in Appendix~\ref{sec:metrics_apdx}.

\noindent
\textbf{Baselines.}
We consider five mitigation baselines. 
Among them, Random Tokens Addition (RTA)~\cite{somepalli2023understanding} mitigates memorization by randomly adding tokens to the original input. 
Ren et al.~\cite{ren2024unveiling} suppresses attention weights corresponding to tokens that trigger memorization. 
Wen et al.~\cite{wen2024detecting} performs optimization on the prompt embedding to steer it away from memorized regions in the embedding space. 
Jeon et al.~\cite{jeon2025understanding} enhances Wen et al.'s method by considering the local curvature while optimizing the initial latent.
Jain et al.~\cite{jain2025classifier} introduces a dynamic guidance scale that pushes the generation trajectory away from memorized content, particularly in the early steps.
As summarized in Table \ref{tab:baseline_comparison}, our method is the only one that adaptively mitigates memorization on the fly with detection. 

\noindent
\textbf{Quantitative Results.}
Fig.~\ref{fig:ablation_curve} visualizes the trade-off between Mem Rate and FID. Detailed results are in Appendix~\ref{sec:more_mitigation_results_apdx}.
Most baselines fail to reach the Pareto optimal region.
Optimization-based methods (Wen et al., Jeon et al.) cluster in the top region with high Mem Rates, as their initial-step mitigation is insufficient to prevent memorization from resurfacing.
Conversely, heuristic methods (RTA and Ren et al.), while reducing Mem Rate further, drift significantly towards the high-FID region (right side), since their coarse-grained interventions disrupt memorization at the cost of degrading quality.
In contrast, our method dominates the bottom-left Pareto front, achieving a \textbf{$0$ Mem Rate} with the lowest FID.
Overall, we observe a positive correlation where mitigation improves quality.
This reflects that memorization is essentially a form of mode collapse. 
Our adaptive mitigation not only eradicates duplicates on-the-fly with negligible computational overhead ($\approx0.01$s/img, see Table~\ref{tab:time_comparison}), but also restores diverse generation, thereby improving fidelity.

\noindent
\textbf{Visualizations.}
Quantitative metrics alone can be deceptive.
First, high CLIP Scores often mask mitigation failures by rewarding the high-quality generation of memorized replicas.
Second, a low Mem Rate does not guarantee success, as it may simply reflect the generation of unrelated content (over-mitigation), characterized by a drastic drop in CLIP Score.
Therefore, visual inspection in Fig.~\ref{fig:exp} is essential.
The results reveal two main failure modes in baselines: either retaining memorized content (most baselines) or suffering from excessive semantic deviation (Jain et al.).
Our method strikes the optimal balance, effectively eliminating memorized content while preserving high semantic fidelity.

\begin{figure}[t]
\centering
\includegraphics[width=\linewidth]{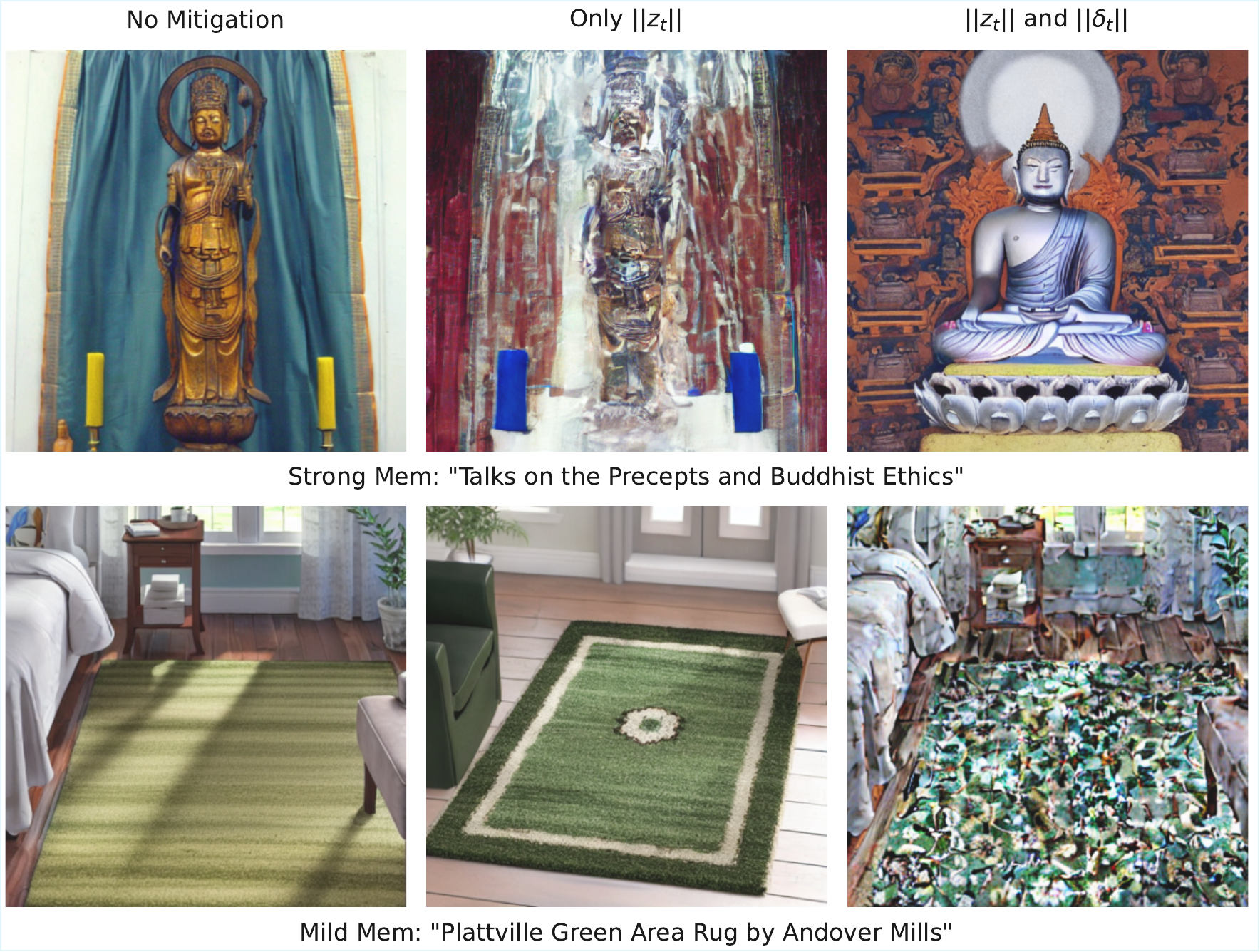} 
\caption{Ablation study of $\|\mathbf{z}_t\|$ and $\|\delta_t\|$ constraints.}
\label{fig:mitigation_ablation}
\end{figure}

\noindent
\textbf{Ablation Study on Stability Constraints.}
We further analyze the effect of different stability constraints on mitigating memorization, as shown in Fig.~\ref{fig:mitigation_ablation}. 
Three types of constraints are considered in our mitigation: the reconstructed latent norm $\|\hat{\mathbf{z}}_0^{(t)}\|$, the latent norm $\|\mathbf{z}_t\|$, and the latent update norm $\|\delta_t\|$. The $\|\hat{\mathbf{z}}_0^{(t)}\|$ constraint is consistently applied as a baseline for stabilizing intermediate reconstructions in both mild and strong memorization cases.

For \textit{mild} memorization, adding the $\|\mathbf{z}_t\|$ constraint effectively removes memorized content while preserving visual quality. 
However, when jointly enforcing both $\|\mathbf{z}_t\|$ and $\|\delta_t\|$, the generation quality significantly deteriorates, and traces of memorization even reappear. This is possibly due to over-constrained dynamics that disrupt the denoising trajectory. 
For \textit{strong} memorization, applying only the $\|\mathbf{z}_t\|$ constraint is insufficient, as it tends to preserve structural patterns of the memorized image and results in poor image quality. In contrast, jointly enforcing both $\|\mathbf{z}_t\|$ and $\|\delta_t\|$ constraints successfully removes memorization and restores high-quality outputs.

These observations highlight the importance of adaptively selecting constraint combinations based on the severity of memorization. Our multi-constraint design enables effective mitigation across different scenarios.

\noindent
\textbf{Case Study.}
Fig.~\ref{fig:case_study} underscores the necessity of on-the-fly mitigation.
Static methods like RTA and Wen et al. apply constraints only at the initialization stage (via token addition or embedding optimization).
As the visualization reveals, these early constraints often fail to persist, allowing the denoising trajectory to drift back towards memorized attractors in subsequent steps.
In contrast, our framework monitors the process in real-time and actively intercepts these drifts, adaptively mitigating whenever memorization signals resurface to ensure a unique generation.
Additional cases are provided in Appendix~\ref{sec:more_case_study_apdx}.

\begin{figure}[t]
\centering
\includegraphics[width=\linewidth]{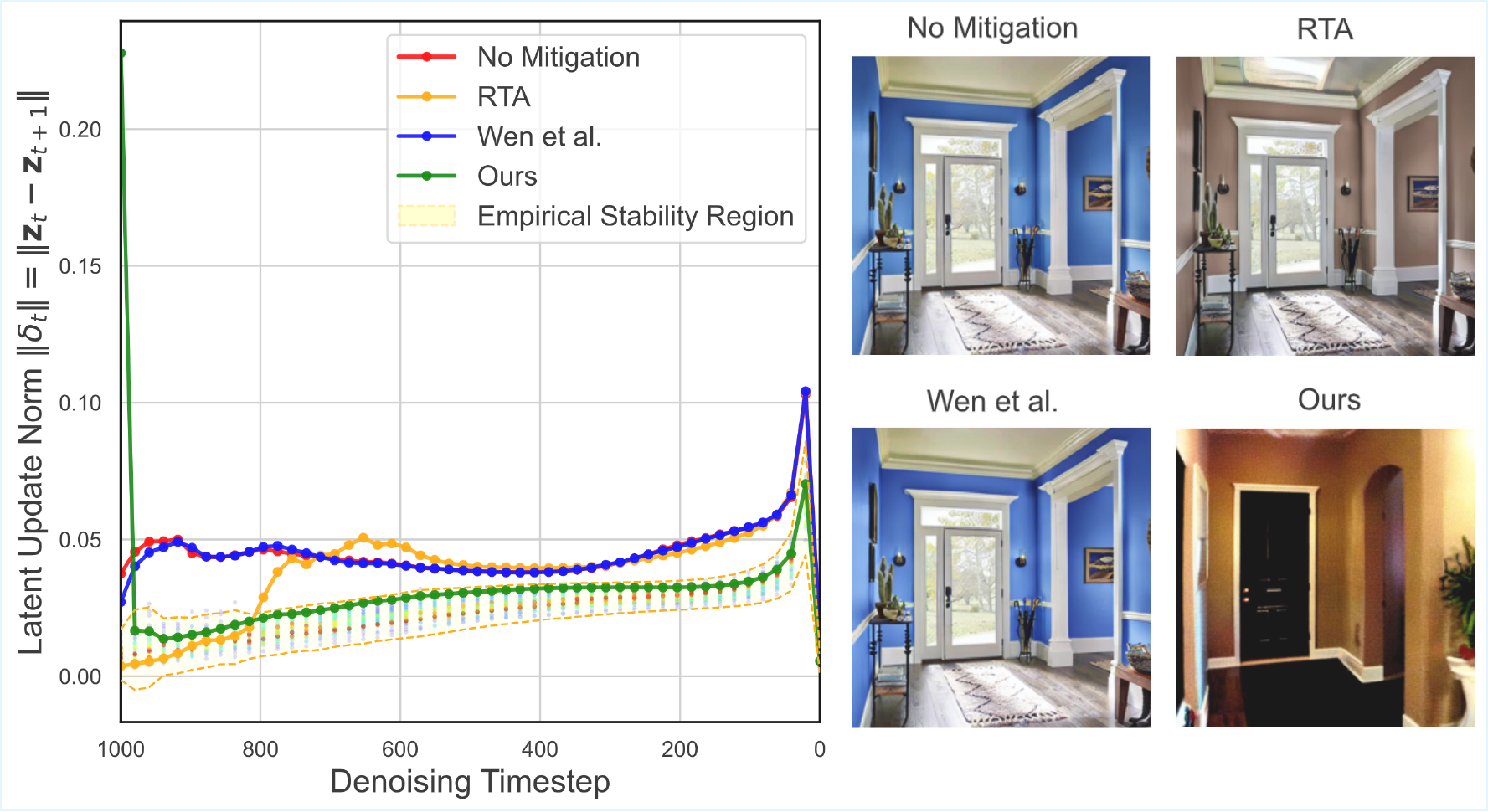}
\caption{Comparison of latent update trajectories and generated images on SD 1.4 with memorized prompts.}
\label{fig:case_study}
\end{figure}

\section{Conclusion}
In this work, we uncovered a fundamental link between memorization and numerical instability in diffusion models. 
By analyzing the latent dynamics, we showed that memorized prompts drive the generation trajectory out of the \textit{empirical stability regions} characteristic of normal samples, which is a phenomenon unified across different samplers (e.g., PNDM and DDIM). 
Guided by this theoretical insight, we proposed a real-time framework for memorization detection and mitigation. 
Our method monitors stability deviations step-by-step and applies adaptive trajectory corrections on the fly, eliminating the need for costly re-generation. 
Extensive experiments confirm that our approach not only achieves superior detection accuracy and complete memorization suppression but also preserves image quality with negligible computational overhead. 
We believe this stability-centric perspective offers a new, principled direction for building trustworthy generative models.

\section*{Acknowledgements}
We would like to thank the anonymous reviewers for their insightful comments that helped improve the quality of the paper. This work was supported in part by the National Natural Science Foundation of China (62472096, 62502157). 
Min Yang is a faculty of Shanghai Pudong Research Institute of Cryptology, and Engineering Research Center of Cyber Security Auditing and Monitoring, Ministry of Education, China.
Mi Zhang and Min Yang are the corresponding authors. 



\bibliographystyle{ACM-Reference-Format}
\bibliography{main}

\appendix
\section{Detailed Proofs}  \label{sec:proof}

\setcounter{theorem}{0}

In this section, we provide detailed proofs for Theorem \ref{thm:normal-stability} and \ref{thm:z0-stability} presented in the main text.

\begin{theorem}[Normal Trajectories Stability]
Let $\delta_t$ denote the latent updates of normal prompts, and let $\mathcal{R}_{\delta}^{(t)}$ be the empirical stability region defined in Def. \ref{def:stability-region}. Then with high probability, we have:
\[
\|\delta_t\| \in \mathcal{R}_{\delta}^{(t)}, \quad \forall t \in \{1, \dots, T\}.
\]
In other words, latent trajectories from non-memorized prompts tend to stay within their corresponding empirical stability region throughout the generation.
\end{theorem}

\begin{proof}[Sketch]
The empirical region $\mathcal{R}_{\delta}^{(t)}$ is constructed from the empirical mean $\mu_t$ and std. $\sigma_t$ over $\|\delta_t\|$ sampled from the training set. Under the assumption that these updates follow a sub-Gaussian distribution \cite{ho2020denoising,song2020score}, concentration bounds ensure that most values lie within $[\mu_t \pm \gamma \sigma_t]$. Thus, the definition of $\mathcal{R}_{\delta}^{(t)}$ captures the high-probability support of normal trajectories.
\end{proof}

\begin{proof}
Let $\delta_t := \mathbf{z}_t - \mathbf{z}_{t+1}$ denote the latent update at step $t$ for a normal prompt. By the formulation of diffusion sampling, the update $\delta_t$ is determined by the denoising network $f_\theta$ and the scheduler (e.g., PNDM, DDIM, taking one-step Euler method as an example here for simplicity):

\[
\mathbf{z}_{t+1} = \mathbf{z}_t - \Delta t \cdot f_\theta(\mathbf{z}_t, t, c)
\]
\[
\Rightarrow \delta_t = \mathbf{z}_t - \mathbf{z}_{t+1} = \Delta t \cdot f_\theta(\mathbf{z}_t, t, c).
\]

For normal prompts, the output of $f_\theta$ is governed by the learned distribution of the training data and is assumed to be sub-Gaussian~\cite{ho2020denoising,song2020score}. That is, for each $t$, the random variable $\|\delta_t\|$ satisfies the sub-Gaussian tail bound:
\[
\mathbb{P}\left( |\|\delta_t\| - \mu_t| \geq \gamma \sigma_t \right) \leq 2 \exp\left( -\frac{\gamma^2}{2} \right),
\]
where $\mu_t$ and $\sigma_t$ are the empirical mean and standard deviation of $\|\delta_t\|$ over normal prompts at step $t$, and $\gamma$ is a positive constant.

Therefore, for a sufficiently large $\gamma$ (e.g., $\gamma=3$), the probability that $\|\delta_t\|$ falls outside the interval $[\mu_t - \gamma \sigma_t,\, \mu_t + \gamma \sigma_t]$ is exponentially small. By the union bound over all $T$ steps, the probability that any $\|\delta_t\|$ deviates from $\mathcal{R}_{\delta}^{(t)}$ is at most $2T \exp(-\gamma^2/2)$, which remains small for moderate $T$ and $\gamma$.

Thus, with high probability, for all $t \in \{1, \dots, T\}$,
\[
\|\delta_t\| \in \mathcal{R}_{\delta}^{(t)} = [\mu_t - \gamma \sigma_t,\, \mu_t + \gamma \sigma_t]
\]
as claimed.

This completes the proof that latent trajectories from non-memorized prompts tend to stay within their corresponding empirical stability region throughout the generation process.
\end{proof}

\noindent
\textbf{Remark on PNDM and Multi-Step Solvers:}
The above proof is based on the single-step update formulation. In practice, multi-step solvers such as PNDM (AB4) \cite{liu2022pseudonumericalmethodsdiffusion} compute the latent update $\delta_t$ as a weighted combination of denoising predictions from several previous steps:
\[
\delta_t = \Delta t \cdot \sum_{i=0}^{3} b_i f_\theta(\mathbf{z}_{t-i}, t-i, c)
\]
This "transfer part" introduces momentum and extrapolation effects, which can amplify local irregularities in the latent trajectory. However, as long as the denoising outputs $f_\theta$ for normal prompts remain sub-Gaussian and bounded, the resulting $\delta_t$ still concentrates within the empirical stability region $\mathcal{R}_{\delta}^{(t)}$ with high probability. For memorized prompts, the multi-step combination may further exaggerate deviations, making instability more detectable.

\begin{proposition}[Instability of Memorized Trajectories]
Let $c_{mem}$ be a memorized prompt where the conditional noise prediction norm significantly dominates the unconditional one, such that $\|\epsilon_\theta(\mathbf{z}_t, t, c_{mem})\| \ge \beta \|\epsilon_\theta(\mathbf{z}_t, t, \phi)\|$ for a large factor $\beta > 1$.
With guidance scale $w$, the magnitude of the latent update $\|\delta_t\|$ for the memorized prompt satisfies:
\[
\|\delta_t\|_{mem} \approx w \cdot C \cdot \|\epsilon_\theta(\mathbf{z}_t, t, c_{mem})\|,
\]
where $C$ is a constant related to the step size. Consequently, due to the excessively large gradient norm $\|\epsilon_\theta(\mathbf{z}_t, t, c_{mem})\|$ (large $\beta$), the resulting update $\|\delta_t\|_{mem}$ exceeds the upper bound of the stability region $\mathcal{R}_{\delta}^{(t)} = [\mu_t \pm \gamma \sigma_t]$ calibrated on normal samples, causing the trajectory to diverge from the stable manifold.
\end{proposition}

\begin{proof}
By definition of classifier-free guidance (CFG), the final noise prediction $\epsilon_\theta^{\text{CFG}}$ is given by:
\[
\epsilon_\theta^{\text{CFG}}(\mathbf{z}_t, t, c) = \epsilon_\theta(\mathbf{z}_t, t, \phi) + w \cdot \left(\epsilon_\theta(\mathbf{z}_t, t, c) - \epsilon_\theta(\mathbf{z}_t, t, \phi)\right).
\]
Rearranging terms, we can express it as a weighted sum:
\[
\epsilon_\theta^{\text{CFG}}(\mathbf{z}_t, t, c) = w \cdot \epsilon_\theta(\mathbf{z}_t, t, c) + (1-w) \cdot \epsilon_\theta(\mathbf{z}_t, t, \phi).
\]
Now, we consider the magnitude of this vector. Using the triangle inequality properties ($|\|a\| - \|b\|| \le \|a-b\|$), we can bound the norm from below:
\begin{align*}
\|\epsilon_\theta^{\text{CFG}}\| &= \| w \cdot \epsilon_\theta(\mathbf{z}_t, t, c) - (w-1) \cdot \epsilon_\theta(\mathbf{z}_t, t, \phi) \| \\
&\ge w \|\epsilon_\theta(\mathbf{z}_t, t, c)\| - (w-1) \|\epsilon_\theta(\mathbf{z}_t, t, \phi)\|,
\end{align*}
assuming $w > 1$.

We invoke the memorization assumption: for a memorized prompt $c_{mem}$, the model is overfitted, resulting in gradients (noise predictions) that are significantly sharper than the unconditional priors. Specifically, we assume $\|\epsilon_\theta(\mathbf{z}_t, t, c_{mem})\| \ge \beta \|\epsilon_\theta(\mathbf{z}_t, t, \phi)\|$ for some large factor $\beta > 1$.
Substituting $\|\epsilon_\theta(\mathbf{z}_t, t, \phi)\| \le \frac{1}{\beta} \|\epsilon_\theta(\mathbf{z}_t, t, c_{mem})\|$ into the lower bound:
\begin{align*}
\|\epsilon_\theta^{\text{CFG}}\| &\ge w \|\epsilon_\theta(\mathbf{z}_t, t, c_{mem})\| - (w-1) \frac{1}{\beta} \|\epsilon_\theta(\mathbf{z}_t, t, c_{mem})\| \\
&= \left( w - \frac{w-1}{\beta} \right) \|\epsilon_\theta(\mathbf{z}_t, t, c_{mem})\|.
\end{align*}
For memorization, $\beta$ is large (i.e., the conditional signal overwhelmingly dominates the unconditional noise). As $\beta \to \infty$, the term $\frac{w-1}{\beta} \to 0$, leading to the approximation:
\[
\|\epsilon_\theta^{\text{CFG}}\| \approx w \|\epsilon_\theta(\mathbf{z}_t, t, c_{mem})\|.
\]
Next, relating this to the latent update $\delta_t$. For a generic solver step, the update is proportional to the noise prediction:
\[
\delta_t = \mathbf{z}_t - \mathbf{z}_{t+1} = \Delta t \cdot \epsilon_\theta^{\text{CFG}}. 
\]
Thus, the magnitude $\|\delta_t\|$ is:
\[
\|\delta_t\|_{mem} = \Delta t \cdot \|\epsilon_\theta^{\text{CFG}}\| \approx \Delta t \cdot w \cdot \|\epsilon_\theta(\mathbf{z}_t, t, c_{mem})\|.
\]
Note that for multi-step solvers like PNDM, $\delta_t$ is a linear combination of past $\epsilon$'s, but the dominant term is still proportional to the current scale, so we use a constant $C$ in the proposition statement to subsume $\Delta t$ and solver coefficients.

Finally, we consider the stability region $\mathcal{R}_{\delta}^{(t)}$. Its upper bound, $B_{max} = \mu_t + \gamma \sigma_t$, is determined by the statistics of \textit{normal} prompts. 
For normal prompts, the conditional gradient $\|\epsilon_\theta(\mathbf{z}_t, t, c_{norm})\|$ is bounded, resulting in a stable reference distribution for $\delta_t$.

In contrast, for memorized prompts, the overfitting leads to an intrinsically much larger gradient norm: $\|\epsilon_\theta(\mathbf{z}_t, t, c_{mem})\| \gg \|\epsilon_\theta(\mathbf{z}_t, t, c_{norm})\|$.
Although the same scale $w$ is applied, it amplifies this pre-existing disparity. 
Since the stability threshold is calibrated on the smaller $\|\epsilon_{norm}\|$, the trajectory driven by the significantly larger $\|\epsilon_{mem}\|$ will inevitably exceed this threshold.
Thus, the violation of the stability region is primarily driven by the sharpness of the memorized minimum (large $\beta$), which is then scaled up by $w$ to cross the boundary defined by normal statistics.
\end{proof}

\begin{theorem}[Normal $\hat{\mathbf{z}}_0$ Stability]
Let $\hat{\mathbf{z}}_0^{(t)}$ be the reconstructed initial latent at step $t$ in DDIM, and let $\mathcal{R}_0^{(t)}$ be the stability region in Def. \ref{def:z0-stability-region}. Then, for normal generations, we have:
\[
\hat{\mathbf{z}}_0^{(t)} \in \mathcal{R}_0^{(t)}, \quad \forall t \in \{1, \dots, T\}.
\]
\end{theorem}

\begin{proof}[Sketch]
The DDIM formulation deterministically maps $\mathbf{z}_t$ and $\epsilon_\theta$ to $\hat{\mathbf{z}}_0^{(t)}$. For normal prompts, $\epsilon_\theta$ behaves regularly, leading to a stable and bounded reconstruction $\hat{\mathbf{z}}_0^{(t)}$. 
\end{proof}

\begin{proof}
Recall that in DDIM, the reconstructed initial latent at step $t$ is given by
\[
\hat{\mathbf{z}}_0^{(t)} = \frac{1}{\sqrt{\alpha_t}} \left( \mathbf{z}_t - \sqrt{1-\alpha_t} \cdot \epsilon_\theta(\mathbf{z}_t, t, c) \right),
\]
where $\alpha_t$ is a fixed schedule parameter and $\epsilon_\theta$ is the denoising prediction.

For normal prompts, the output $\epsilon_\theta(\mathbf{z}_t, t, c)$ is governed by the learned distribution and is assumed to be sub-Gaussian, as in the previous theorem. Thus, for each $t$, the random variable $\|\hat{\mathbf{z}}_0^{(t)}\|$ is also sub-Gaussian, since it is an affine transformation of sub-Gaussian variables.

Let $\mu_0^{(t)}$ and $\sigma_0^{(t)}$ denote the empirical mean and standard deviation of $\|\hat{\mathbf{z}}_0^{(t)}\|$ over normal prompts at step $t$. By the sub-Gaussian tail bound,
\[
\mathbb{P}\left( |\|\hat{\mathbf{z}}_0^{(t)}\| - \mu_0^{(t)}| \geq \gamma \sigma_0^{(t)} \right) \leq 2 \exp\left( -\frac{\gamma^2}{2} \right),
\]
where $\gamma$ is a positive constant.

Therefore, for a sufficiently large $\gamma$ (e.g., $\gamma=3$), the probability that $\|\hat{\mathbf{z}}_0^{(t)}\|$ falls outside the interval $[\mu_0^{(t)} - \gamma \sigma_0^{(t)},\, \mu_0^{(t)} + \gamma \sigma_0^{(t)}]$ is exponentially small. By the union bound over all $T$ steps, the probability that any $\|\hat{\mathbf{z}}_0^{(t)}\|$ deviates from $\mathcal{R}_0^{(t)}$ is at most $2T \exp(-\gamma^2/2)$, which remains small for moderate $T$ and $\gamma$.

Thus, with high probability, for all $t \in \{1, \dots, T\}$,
\[
\|\hat{\mathbf{z}}_0^{(t)}\| \in \mathcal{R}_0^{(t)} = [\mu_0^{(t)} - \gamma \sigma_0^{(t)},\, \mu_0^{(t)} + \gamma \sigma_0^{(t)}]
\]
as claimed.

This completes the proof that reconstructed latents from non-memorized prompts tend to stay within their corresponding empirical stability region throughout the generation process in DDIM.
\end{proof}

\section{Theoretical Connections to Wen's Metric}  \label{sec:relation_wen}

Our work introduces a trajectory stability framework that unifies and extends prior empirical observations on inference guidance. Here, we clarify how our statistical boundaries theoretically relate to, yet practically transcend, the guidance magnitude metric proposed by Wen et al.~\cite{wen2024detecting}.

\noindent
\textbf{Detection: From Absolute Heuristics to Statistical Outliers. }
Wen et al. define their detection signal based on the magnitude of the conditional guidance term: $M_{\text{Wen}} = \|\epsilon_\theta(\mathbf{z}_t, c) - \epsilon_\theta(\mathbf{z}_t, \phi)\|_2$.
In diffusion solvers, the latent update vector $\delta_t$ is locally proportional to this CFG output:
\[
\delta_t \approx \Delta t \cdot \epsilon_\theta^{\text{CFG}} = \Delta t \cdot \left[ \epsilon_\theta(\mathbf{z}_t, \phi) + w \cdot (\epsilon_\theta(\mathbf{z}_t, c) - \epsilon_\theta(\mathbf{z}_t, \phi)) \right].
\]
For memorized generations, the guidance term dominates the update magnitude. Thus, we have the approximation:
\[
\|\delta_t\| \approx \Delta t \cdot w \cdot \|\epsilon_\theta(\mathbf{z}_t, c) - \epsilon_\theta(\mathbf{z}_t, \phi)\| = C \cdot M_{\text{Wen}},
\]
where $C$ is a scaling factor involving the step size and guidance scale.

While this confirms that both our method and Wen's track the same underlying ``over-guidance'' force, the detection paradigms are fundamentally different.
Wen et al. treat detection as an \textit{absolute thresholding} problem on $M_{\text{Wen}}$. Since $M_{\text{Wen}}$ varies drastically across different timesteps and model checkpoints, finding a functional threshold is often heuristic and brittle.
In contrast, we reframe detection as a \textit{distributional outlier} problem. By enforcing the empirical stability region, we do not care about the absolute value of the force, but rather whether this force violates the statistical laws of natural image generation ($s_t = (\|\delta_t\| - \mu_t)/\sigma_t > \gamma$). 
This shift from magnitude to probability makes our method model-agnostic and time-invariant, providing a robust, calibrated standard for what constitutes an anomaly.

\noindent
\textbf{Mitigation: Global Suppression vs. Manifold Projection. }
The distinction is even more critical in mitigation. 
Prior work typically analyzes memorization in the noise space ($\epsilon$). Our introduction of the reconstructed initial latent $\hat{\mathbf{z}}_0$ maps these deviations into the \textit{signal space}, revealing the semantic impact of memorization.
Substituting the CFG formula into the definition of $\hat{\mathbf{z}}_0$, we obtain:
\[
\hat{\mathbf{z}}_0^{(t)} = \frac{\mathbf{z}_t}{\sqrt{\bar{\alpha}_t}} - \sqrt{\frac{1 - \bar{\alpha}_t}{\bar{\alpha}_t}} \left( \epsilon_\theta(\mathbf{z}_t, \phi) + w \cdot (\epsilon_\theta(\mathbf{z}_t, c) - \epsilon_\theta(\mathbf{z}_t, \phi)) \right).
\]
Let $\hat{\mathbf{z}}_{0, \phi}$ be the reconstruction based solely on the unconditional prediction. The deviation caused by the guidance term is:
\[
\hat{\mathbf{z}}_0^{(t)} - \hat{\mathbf{z}}_{0, \phi} = - w \sqrt{\frac{1 - \bar{\alpha}_t}{\bar{\alpha}_t}} \left( \epsilon_\theta(\mathbf{z}_t, c) - \epsilon_\theta(\mathbf{z}_t, \phi) \right).
\]
Taking the norm, we see the geometric interpretation of Wen's metric in the image manifold:
\[
\|\Delta \hat{\mathbf{z}}_0\| \propto w \cdot M_{\text{Wen}}.
\]
This relation shows that excessive guidance pushes the estimated image $\hat{\mathbf{z}}_0$ off the natural image manifold.
Wen et al. mitigate this by suppressing the guidance signal via iterative prompt embedding optimization.
Such a strategy acts as signal attenuation: they dampen the force driving the generation to avoid memorization. However, this comes with the heavy cost of optimization latency (``detect-then-retry'') and potential under- or over-mitigation from embedding modification only at the first timestep without an appropriate threshold on $M_{\text{Wen}}$.
Our approach, conversely, performs surgical correction via on-the-fly adaptive projection throughout the entire generation process. Instead of attenuating the guidance force blindly, we allow it to remain strong as long as the trajectory is within the stability region. 
Only when $\hat{\mathbf{z}}_0$ breaches the stability boundary do we intervene, clamping it specifically to the boundary of the normal distribution ($\mathcal{R}_0^{(t)}$). 
This strategy essentially retains the maximum safe guidance, preserving generation quality and semantic detail without the overhead of optimization.

\section{Algorithms}  \label{sec:algorithm}
In this section, we present the algorithm used for memorization detection and mitigation based on the empirical stability regions $\mathcal{R}_{\delta}^{(t)}$, $\mathcal{R}_0^{(t)}$ and $\mathcal{R}_z^{(t)}$. 

Algorithm~\ref{alg:detection} describes the memorization detection process, which computes the latent updates $\delta_t$ during the diffusion sampling process and evaluates their stability against the empirical stability regions. The maximum Z-score across the first $s$ denoising steps is returned as the memorization score $S_{\text{mem}}$. In practice, $s$ can be set from $0$ to $T$, depending on the desired detection sensitivity. Empirical results show that $s = 1$ can already achieve AUC $>0.99$ with SD 1.4, PNDM scheduler. Note that $s_t$ at each timestep $t$ itself can also be used to determine whether a generation is subject to memorization. 

\begin{algorithm}[ht]
    \caption{Memorization Detection via Empirical Stability Region on Latent Updates $\mathcal{R}_{\delta}^{(t)}$.}
    \label{alg:detection}
    \begin{algorithmic}[1]
    \Require{Prompt $c$, diffusion model $f_\theta$, empirical stability regions $\{\mathcal{R}_{\delta}^{(t)}\}_{t=1}^T$ (for $\delta_t$), total steps $T$, detection steps $s$}
    \Ensure{Memorization score $S_{\text{mem}}$}

    \State Initialize latent $\mathbf{z}_T$ from noise
    \For{$t = T-1$ to $T-s$}
        \State Sample noise with diffusion model $f_\theta(\mathbf{z}_{t+1}, t, c)$
        \State Generate $\mathbf{z}_t$ using sampler (e.g., DDIM/PNDM)
        \State Compute latent update $\delta_t = \mathbf{z}_t - \mathbf{z}_{t+1}$
        \State Compute Z-score $s_t = \left| \frac{\|\delta_t\| - \mu_t}{\sigma_t} \right|$
    \EndFor
    \State $S_{\text{mem}} \gets \max_{t} s_t$
    \State \Return $S_{\text{mem}}$ 
    \end{algorithmic}
\end{algorithm}

Algorithm~\ref{alg:mitigation} describes the on-the-fly memorization mitigation process. It uses the empirical stability regions to constrain the latent updates and reconstructed latents during the denoising process. The algorithm first initializes the latent from noise and then iteratively generates latents while checking for memorization based on the Z-scores. 
The detection is based on the core part (line 3 to 6) of Algorithm~\ref{alg:detection}, where it computes the latent updates $\delta_t$ and their Z-scores. When the Z-score exceeds the mild threshold for the first time, it checks if a memorization type is detected based on the reconstructed latent $\hat{\mathbf{z}}_0^{(t)}$. 
In practice, such deviation usually occurs at the first step, most of them happen within the first ten steps.
Once a memorization type is detected, it rescales the reconstructed latent $\hat{\mathbf{z}}_0^{(t)}$ to ensure it remains within the stability region, and optionally constrains the norms of $\mathbf{z}_t$ and $\delta_t$ as needed.

\begin{algorithm}[ht]
    \caption{Memorization Mitigation via On-the-Fly Stability-Constrained Adaptive Sampling}
    \label{alg:mitigation}
    \begin{algorithmic}[1]
    \Require{Prompt $c$, diffusion model $f_\theta$, empirical stability regions $\{\mathcal{R}_{\delta}^{(t)}\}_{t=1}^T$, $\{\mathcal{R}_z^{(t)}\}_{t=1}^T$, $\{\mathcal{R}_0^{(t)}\}_{t=1}^T$, thresholds $\tau_{\text{mild}}, \tau_{\text{strong}}$, mild mitigation steps $k$}
    \Ensure{Generated image $\mathbf{x}_0$}

    \Function{Rescale}{$\mathbf{v}, \mu$}
        \State $r \gets \mu / \|\mathbf{v}\|$
        \State \Return $r \cdot \mathbf{v}$
    \EndFunction

    \State Initialize latent $\mathbf{z}_T$ from noise
    \State $\texttt{mem\_type} \gets$ \texttt{none}
    \For{$t = T-1$ to $0$}
        \State Generate $\mathbf{z}_t$ using sampler (e.g., DDIM/PNDM)
        \State Compute latent update $\delta_t = \mathbf{z}_t - \mathbf{z}_{t+1}$
        \State Compute reconstructed latent $\hat{\mathbf{z}}_0^{(t)}$
        \State Compute Z-scores: $s_t = \left| \frac{\|\delta_t\| - \mu_t}{\sigma_t} \right|$, $s_0^{(t)} = \left| \frac{\|\hat{\mathbf{z}}_0^{(t)}\| - \mu_0^{(t)}}{\sigma_0^{(t)}} \right|$
        \If{$\texttt{mem\_type} = \texttt{none}$ \textbf{and} $s_t > \tau_{\text{mild}}$} \Comment{Set \texttt{mem\_type} based on $s_0^{(t)}$ when $s_t$ first exceeds mild threshold}
            \If{$s_0^{(t)} > \tau_{\text{strong}}$}  
                \State $\texttt{mem\_type} \gets$ \texttt{strong}
            \ElsIf{$s_0^{(t)} > \tau_{\text{mild}}$}
                \State $\texttt{mem\_type} \gets$ \texttt{mild}
            \EndIf
        \EndIf
        \If{$s_t > \tau_{\text{mild}}$}
            \If{$\texttt{mem\_type} = \texttt{strong}$}
                \State $\hat{\mathbf{z}}_0^{(t)} \gets$ \Call{Rescale}{$\hat{\mathbf{z}}_0^{(t)}, \mu_0^{(t)}$}
                \State Optionally constrain $\|\mathbf{z}_t\|$ and $\|\delta_t\|$ within $\mathcal{R}_z^{(t)}$ and $\mathcal{R}_{\delta}^{(t)}$
            \ElsIf{$\texttt{mem\_type} = \texttt{mild}$ \textbf{and} $t < k$}
                \State $\hat{\mathbf{z}}_0^{(t)} \gets$ \Call{Rescale}{$\hat{\mathbf{z}}_0^{(t)}, \mu_0^{(t)}$}
                \State Optionally constrain $\|\mathbf{z}_t\|$ within $\mathcal{R}_z^{(t)}$
            \EndIf
        \EndIf
        \State Compute new latent $\mathbf{z}_t$ \Comment{Update $\mathbf{z}_t$ after rescale}
    \EndFor
    \State Decode final latent $\mathbf{z}_0$ to image $\mathbf{x}_0$
    \State \Return $\mathbf{x}_0$
    \end{algorithmic}
\end{algorithm}

\section{Experiment Settings}  \label{sec:experiment_settings_apdx}

\subsection{Datasets}
\begin{itemize}
    \item \textbf{Memorized Prompts:} For SD 1.4 and 1.5, we use $500$ memorized prompts sampled from~\cite{webster2023reproducible}, which are also used by~\cite{wen2024detecting} for evaluating memorization detection. For SD 2.1, we use $219$ memorized prompts, following the setup of~\cite{ren2024unveiling}.
    \item \textbf{Non-Memorized Prompts:} We follow~\cite{wen2024detecting} to build a $500$ prompts set from the Lexica-art gallery (\path{Gustavosta/Stable-Diffusion-Prompts}), the MS-COCO-2017 validation set and GPT-4 generated prompts. 
    \item \textbf{Reference Prompts:} We sample a total of $500$ non-memorized prompts from the LAION-400M dataset~\cite{schuhmann2021laion} as reference prompts for estimating the stability regions. From these, we randomly select $50$ prompts for all evaluation unless specified. These prompts are distinct from those used in the detection and mitigation tasks.
\end{itemize}

\subsection{Models}
For pretrained models, we use SD 1.4 (\path{CompVis/stable-diffusion-v1-4}), 1.5 (\path{runwayml/stable-diffusion-v1-5}), and 2.1 (\path{stabilityai/stable-diffusion-2-1}) from the HuggingFace model hub. 

For finetuned models, we reuse the SD 1.4 finetuned by~\cite{wen2024detecting} open-sourced in their GitHub repository. They use $200$ LAION data points, each duplicated $200$ times as the memorized set, and $120,000$ distinct LAION data points to retain the generalization ability of the model. The memorized set for this model is therefore the $200$ duplicated prompts used during fine-tuning.

\begin{figure*}[t]
\centering
\includegraphics[width=\linewidth]{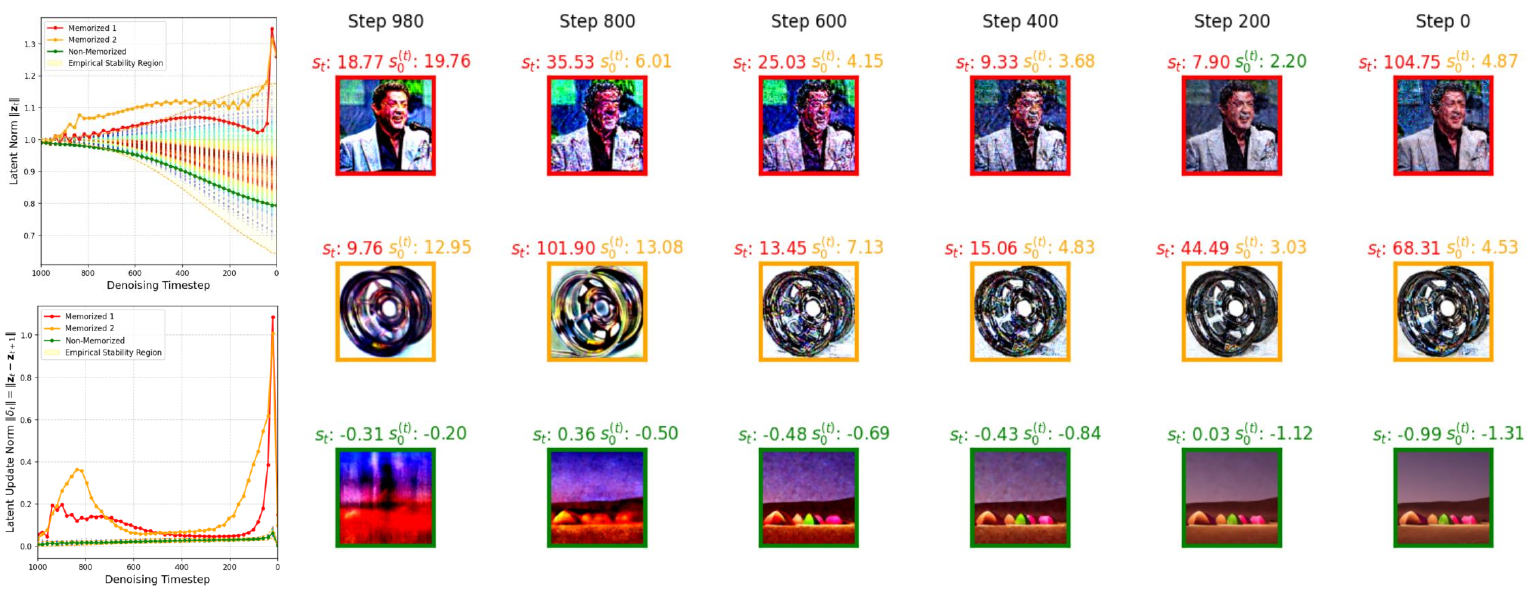}
\caption{PNDM generation process on SD 1.4 using strong/mild/non- memorized prompts.}
\label{fig:case_pndm}
\end{figure*}

\begin{figure*}[t]
\centering
\includegraphics[width=\linewidth]{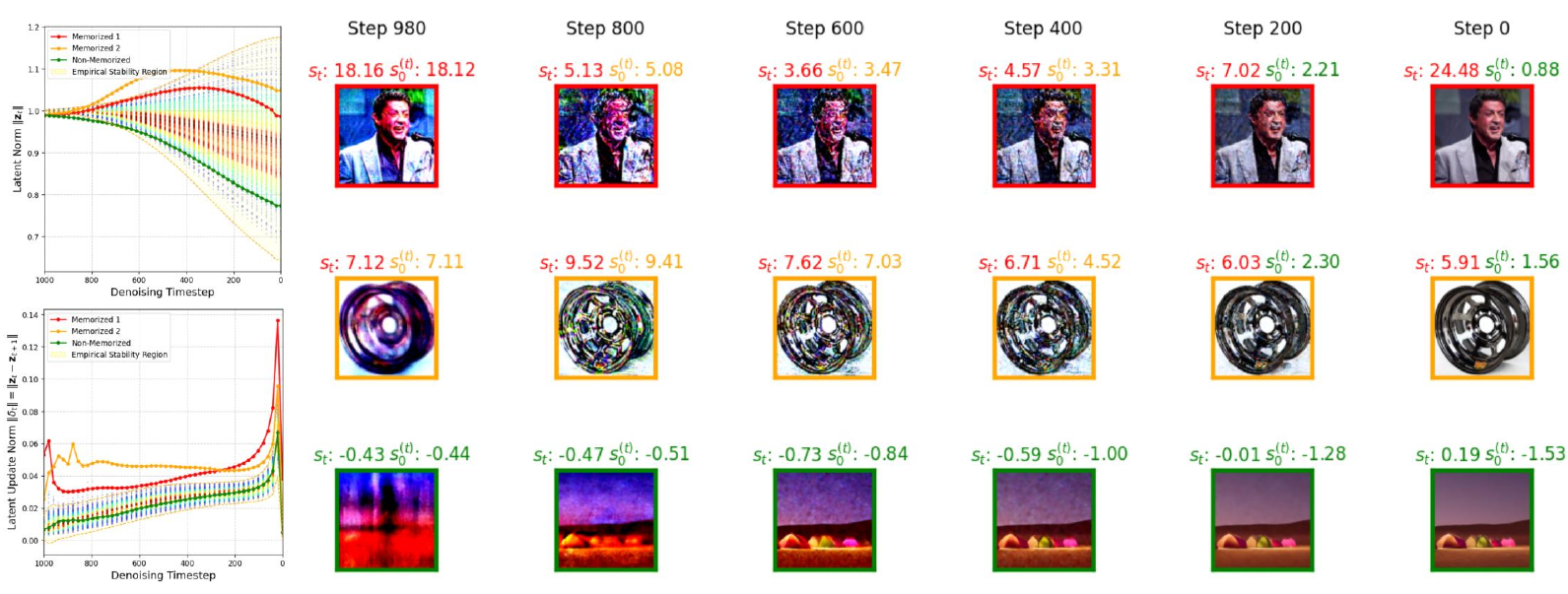}
\caption{DDIM generation process on SD 1.4 using strong/mild/non- memorized prompts.}
\label{fig:case_ddim}
\end{figure*}

\subsection{Metrics}  \label{sec:metrics_apdx}
Following previous works~\cite{somepalli2023understanding,ren2024unveiling,wen2024detecting,jain2025classifier}, we use the following metrics: 
\begin{itemize}
    \item \textbf{AUC} (Area Under the ROC Curve): Measures the overall performance of the detection method over varying thresholds, implemented with \path{scikit-learn}.
    \item \textbf{TPR@$1\%$FPR} (True Positive Rate at 1\% False Positive Rate): Evaluates the detection performance at a low false positive rate, implemented with \path{scikit-learn}.
    \item \textbf{SSCD} (Self-Supervised Copying Detection~\cite{pizzi2022self}): Evaluates the similarity between two images with a contrastive trained model. We use the \path{sscd_disc_large} model parameter provided in the GitHub repository of the original paper, in accordance with~\cite{wen2024detecting}.
    \item \textbf{CLIP Score}: Measures the semantic similarity between generated images and text prompts using CLIP embeddings. We use the \path{openai/clip-vit-base-patch32} model from the \path{transformers} library.
    \item \textbf{FID} (Fréchet Inception Distance): Evaluates the quality and diversity of generated images by comparing their distribution to real images using Inception extracted features. We use the \path{clean-fid} library to compute FID scores, which is a widely used implementation for computing FID.
\end{itemize}

We evaluate CLIP score on memorized prompts and non-memorized prompts separately, and report both scores.
When calculating FID, we use the LAION-400M dataset as the reference distribution, following the setting of~\cite{ren2024unveiling}. The FID is computed as the feature distribution divergence between the set of synthesized images 
and the reference distribution.

\subsection{Baseline Settings}

\begin{itemize}
    \item \textbf{RTA}~\cite{somepalli2023understanding}: Following the implementation in~\cite{wen2024detecting}, we randomly add four tokens for each prompt using the pipeline's tokenizer.
    \item \textbf{Ren et al.}~\cite{ren2024unveiling}: For detection, we adopt the pipeline from the official open-source implementation. For each sample, we record the attention entropy at every step of the denoising process and then compute a final detection score using the metric $D$ proposed in the original paper. For mitigation, we scale or mask the attention scores of the begin and summary tokens within the cross-attention layers following the hyperparameters in the open-source code.
    \item \textbf{Wen et al.}~\cite{wen2024detecting}: We also utilize the official pipeline for this baseline. For detection, the guidance magnitude at each denoising step is saved as the score. For multi-step detection, we follow the source code's approach by averaging the scores from all preceding steps to obtain the final score. For mitigation, we optimize for $10$ steps at the first step of the denoising process, with $l_{target}=3$ according to the settings in the open-source code.
    \item \textbf{Jeon et al.}~\cite{jeon2025understanding}: We use the official open-source code provided by the authors. For detection, we only expand the detection timesteps to the first $3$ steps since the computational cost of this method is very high. Note that the original paper only performs detection at the first step, which should lead to slightly worse detection performance. For mitigation, we follow the default settings with an optimization threshold of $8.2$ and a maximum of $10$ iterations.
    \item \textbf{Jain et al.}~\cite{jain2025classifier}: We dynamically find the first local minimum of the latent update to use as the ``transition point'' following the implementation in the open-source code, and take the negative of the guidance scale before the transition point, i.e., the Opposite-Guidance setting.
\end{itemize}

For all baselines and our proposed method, we set the number of inference steps to $50$ and the guidance scale to $7.5$ following the settings of all these baselines. 

For detection, we conduct evaluations under three distinct settings: a single seed ($0$), four seeds ($0-3$), and eight seeds ($0-7$). In the multi-seed settings, the final scores are calculated by averaging the results from each individual seed. 
For mitigation, to simulate a real-world scenario, we apply a ``detect-then-mitigate'' strategy to the baselines that offer a detection mechanism. The threshold is set to achieve a $1\%$ false positive rate on LAION-400M reference prompts.
For mitigation, we uniformly use the DDIM scheduler and report the results from a fixed random seed $0$.

\subsection{Implementation Details}
Our method consists of two main components: memorization detection and memorization mitigation. 
We mainly rely on the \path{callback_on_step_end} feature of the Stable Diffusion pipeline to implement our method. 
To estimate the empirical stability regions, we record the latent representations and the denoising predictions at each step during the diffusion process. 
Then, for detection, we take the latent representations at each step and compute the distances (Z-scores) to the stability regions.
For mitigation, after each denoising step, we check if the current latent representation is within the stability region. If not, we estimate $\mathbf{z}_0^{(t)}$ with the current latent representation $\mathbf{z}_t$ and the denoising prediction $\epsilon_\theta(\mathbf{z}_t, t, c)$, and then we perform the mitigation in Algorithm~\ref{alg:mitigation}. The resulting latent replaces the original latent when the callback returns.

For hyperparameters, the thresholds $\tau_\text{mild}$ and $\tau_\text{strong}$ are empirically determined based on the statistical distribution of Z-scores from normal prompts. We set $\tau_\text{mild}=3$ and $\tau_\text{strong}=14$, which correspond to the $99.7\%$ and extreme outlier regions under a sub-Gaussian assumption. These values are validated to achieve robust detection and mitigation performance across multiple models and datasets in our experiments.
The mild mitigation steps $k$ is set to $10$, based on empirical analysis of the typical duration of instability in memorized generations with mitigation applied. Previous works~\cite{wen2024detecting,jain2025classifier} have made similar claims that this range is effective for mitigating memorization artifacts without significantly degrading image quality.

All the experiments are conducted on a machine with $8$ NVIDIA RTX 4090 GPUs, each with $24$ GB of VRAM.

\begin{table*}[t]
\centering
\caption{Memorization detection results with AUC, TPR@$1\%$FPR on SD 1.5. The three columns represent detection with the first 3 steps, all steps, and each step (Avg.). The best results are in \textbf{bold}.}
\begin{tabular}{lccccccc}
\toprule
\multicolumn{1}{c}{}                         &                                & \multicolumn{2}{c}{First 3 Steps}                             & \multicolumn{2}{c}{All Steps}           & \multicolumn{2}{c}{Avg. on Steps}       \\ \cline{3-8} 
\multicolumn{1}{c}{\multirow{-2}{*}{Method}} & \multirow{-2}{*}{Num of Seeds} & AUC↑                          & TPR@1\%FPR↑                   & AUC↑               & TPR@1\%FPR↑        & AUC↑               & TPR@1\%FPR↑        \\ \midrule
                                             & 1                              & \multicolumn{2}{c}{}                                          & 0.9609             & 0.6400             & \multicolumn{2}{c}{}                    \\
                                             & 4                              & \multicolumn{2}{c}{}                                          & 0.9412             & 0.6980             & \multicolumn{2}{c}{}                    \\
\multirow{-3}{*}{Ren et al.~\cite{ren2024unveiling}}                 & 8                              & \multicolumn{2}{c}{\multirow{-3}{*}{-}}                       & 0.9413             & 0.6940             & \multicolumn{2}{c}{\multirow{-3}{*}{-}} \\ \midrule
                                             & 1                              & \textbf{0.9827}               & 0.8360               & 0.9737             & \textbf{0.9020}    & 0.9776             & \textbf{0.8964}    \\
                                             & 4                              & 0.9954                        & 0.9760                        & 0.9951             & \textbf{0.9780}    & 0.9958             & \textbf{0.9762}    \\
\multirow{-3}{*}{Wen et al.~\cite{wen2024detecting}}                 & 8                              & 0.9984                        & {0.9860}               & 0.9976             & 0.9860             & 0.9981             & \textbf{0.9857}    \\ \hline
                                             & 1                              & 0.9782                        & \textbf{0.8580}                        & \multicolumn{2}{c}{}                    & \multicolumn{2}{c}{}                    \\
                                             & 4                              & {0.9943} & {0.9760} & \multicolumn{2}{c}{}                    & \multicolumn{2}{c}{}                    \\
\multirow{-3}{*}{Jeon et al.~\cite{jeon2025understanding}}                & 8                              & \textbf{0.9994} & \textbf{0.9900} & \multicolumn{2}{c}{\multirow{-3}{*}{-}} & \multicolumn{2}{c}{\multirow{-3}{*}{-}} \\ \midrule
                                             & 1                              & 0.9789                        & 0.8140                        & \textbf{0.9780}    & 0.8620             & \textbf{0.9798}    & 0.8570             \\
                                             & 4                              & \textbf{0.9965}               & \textbf{0.9800}               & \textbf{0.9970}    & \textbf{0.9780}    & \textbf{0.9967}    & 0.9732             \\
\multirow{-3}{*}{Ours}                       & 8                              & {0.9989}               & 0.9840                        & \textbf{0.9997}    & \textbf{0.9920}    & \textbf{0.9993}    & 0.9849             \\ \bottomrule
\end{tabular}
\label{tab:detect_sd1.5_main}
\end{table*}

\begin{table*}[t]
\centering
\caption{Memorization detection results with AUC, TPR@$1\%$FPR on SD 2.1. The three columns represent detection with the first 3 steps, all steps, and each step (Avg.). The best results are in \textbf{bold}.}
\begin{tabular}{lccccccc}
\toprule
\multicolumn{1}{c}{\multirow{2}{*}{Method}} & \multirow{2}{*}{Num of Seeds} & \multicolumn{2}{c}{First 3 Steps}      & \multicolumn{2}{c}{All Steps}          & \multicolumn{2}{c}{Avg. on Steps}      \\ \cline{3-8} 
\multicolumn{1}{c}{}                        &                               & AUC↑               & TPR@1\%FPR↑       & AUC↑               & TPR@1\%FPR↑       & AUC↑               & TPR@1\%FPR↑       \\ \midrule
\multirow{3}{*}{Ren et al.~\cite{ren2024unveiling}}                 & 1                             & \multicolumn{2}{c}{\multirow{3}{*}{-}} & 0.9343             & 0.4566            & \multicolumn{2}{c}{\multirow{3}{*}{-}} \\
                                            & 4                             & \multicolumn{2}{c}{}                   & 0.9461             & 0.5297            & \multicolumn{2}{c}{}                   \\
                                            & 8                             & \multicolumn{2}{c}{}                   & 0.9477             & 0.6393            & \multicolumn{2}{c}{}                   \\ \midrule
\multirow{3}{*}{Wen et al.~\cite{wen2024detecting}}                 & 1                             & \textbf{0.9752}    & 0.4475            & 0.9747             & \textbf{0.9589}   & \textbf{0.9717}    & \textbf{0.8612}   \\
                                            & 4                             & \textbf{0.9952}    & \textbf{0.9406}   & 0.9914             & \textbf{0.9817}   & 0.9909             & \textbf{0.9763}   \\
                                            & 8                             & \textbf{0.9967}    & \textbf{0.9680}   & 0.9931             & \textbf{0.9863}   & 0.9937             & \textbf{0.9824}   \\ \midrule
\multirow{3}{*}{Jeon et al.~\cite{jeon2025understanding}}                & 1                             & 0.9429             & \textbf{0.6301}   & \multicolumn{2}{c}{\multirow{3}{*}{-}} & \multicolumn{2}{c}{\multirow{3}{*}{-}} \\
                                            & 4                             & 0.9839             & 0.8356            & \multicolumn{2}{c}{}                   & \multicolumn{2}{c}{}                   \\
                                            & 8                             & 0.9909             & 0.8630            & \multicolumn{2}{c}{}                   & \multicolumn{2}{c}{}                   \\ \midrule
\multirow{3}{*}{Ours}                       & 1                             & 0.9622             & 0.4247            & \textbf{0.9773}    & 0.8447            & 0.9631             & 0.7139            \\
                                            & 4                             & 0.9935             & 0.9087            & \textbf{0.9951}    & 0.9589            & \textbf{0.9939}    & 0.9500            \\
                                            & 8                             & 0.9952             & 0.9178            & \textbf{0.9960}    & 0.9726            & \textbf{0.9962}    & 0.9645            \\ \bottomrule
\end{tabular}
\label{tab:detect_sd2.1_main}
\end{table*}

\section{More Experiment Results}  \label{sec:more_experiment_results_apdx}

\subsection{Degraded Generations with Memorized Prompts}

To further elucidate our finding that memorized prompts lead to degraded generation, we present additional degradation examples.

In Fig.~\ref{fig:case_pndm}, we illustrate the PNDM generation process for strong, mild and non-memorized prompts, represented in red, orange, and green, respectively. For memorized prompts, degradation occurs throughout the entire denoising process, whereas the generation quality for non-memorized prompts remains normal. As can be seen, both the latent norm and latent update norm for memorized prompts deviate from the stability region. This deviation is more pronounced in the latter, demonstrating that our proposed detection metric provides a strong signal for identifying memorization.

In Fig.~\ref{fig:case_ddim}, we present the generation results for the same prompts using the DDIM scheduler. Although the final images generated by DDIM do not exhibit degradation, the intermediate results suffer from degradation because the latent norm and latent update norm temporarily deviate from the stability region during the process. This example illustrates that the latent update can be used as a general signal for detecting memorization, even when the final image quality is not degraded.

\subsection{Distribution and Scale of Reference Prompts} \label{sec:ref_prompt_dist_apdx}
To investigate whether detection performance depends on the distribution of reference prompts, we analyze the norm distributions during inference using reference prompts from various data sources, as shown in Fig.~\ref{fig:non_mem_stability_regions_more_datasets}. 
We observe that the distributions of both $\mathbf{z}_t$ and $\delta_t$ norms span similar ranges across different sources, indicating that our method is robust to the choice of reference prompts. 
Among these sources, we select prompts from LAION-400M, as it is the training set for SD and thus most accessible to model providers in practice.

Furthermore, varying the scale of reference prompts also yields similar norm distributions, as illustrated in Fig.~\ref{fig:sample_size_stability_regions}. 
We further evaluate detection performance with different numbers of reference prompts in Fig.~\ref{fig:auc} and~\ref{fig:tpr_at_1_fpr}. 
The results demonstrate that a moderate number of prompts is sufficient to achieve strong detection performance. 
To balance detection accuracy and computational cost, we use $50$ reference prompts in all experiments.

\begin{figure}[t]
\centering
\includegraphics[width=\linewidth]{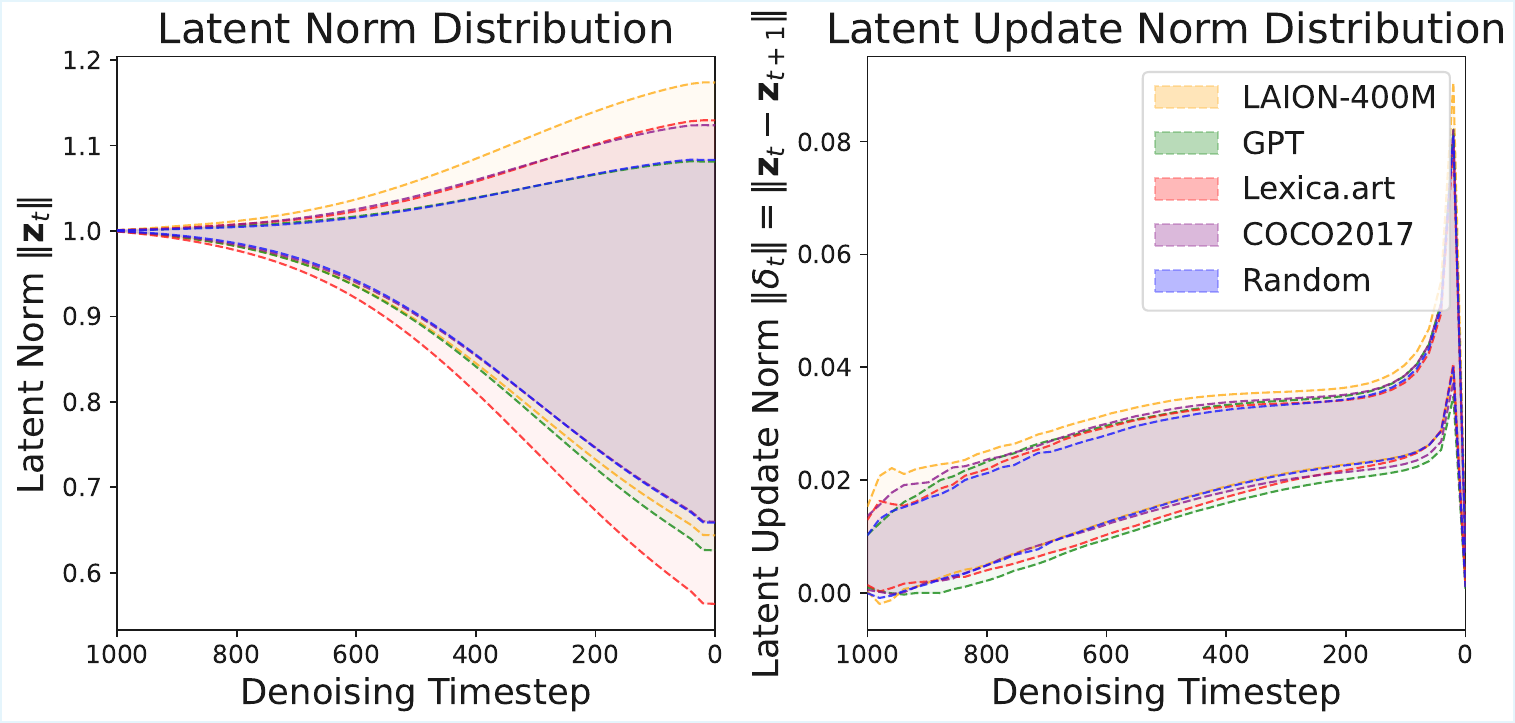}  
\caption{Similar stability regions by prompts from different non-memorized datasets.}
\label{fig:non_mem_stability_regions_more_datasets}
\end{figure}

\begin{figure}[t]
\centering
\includegraphics[width=\linewidth]{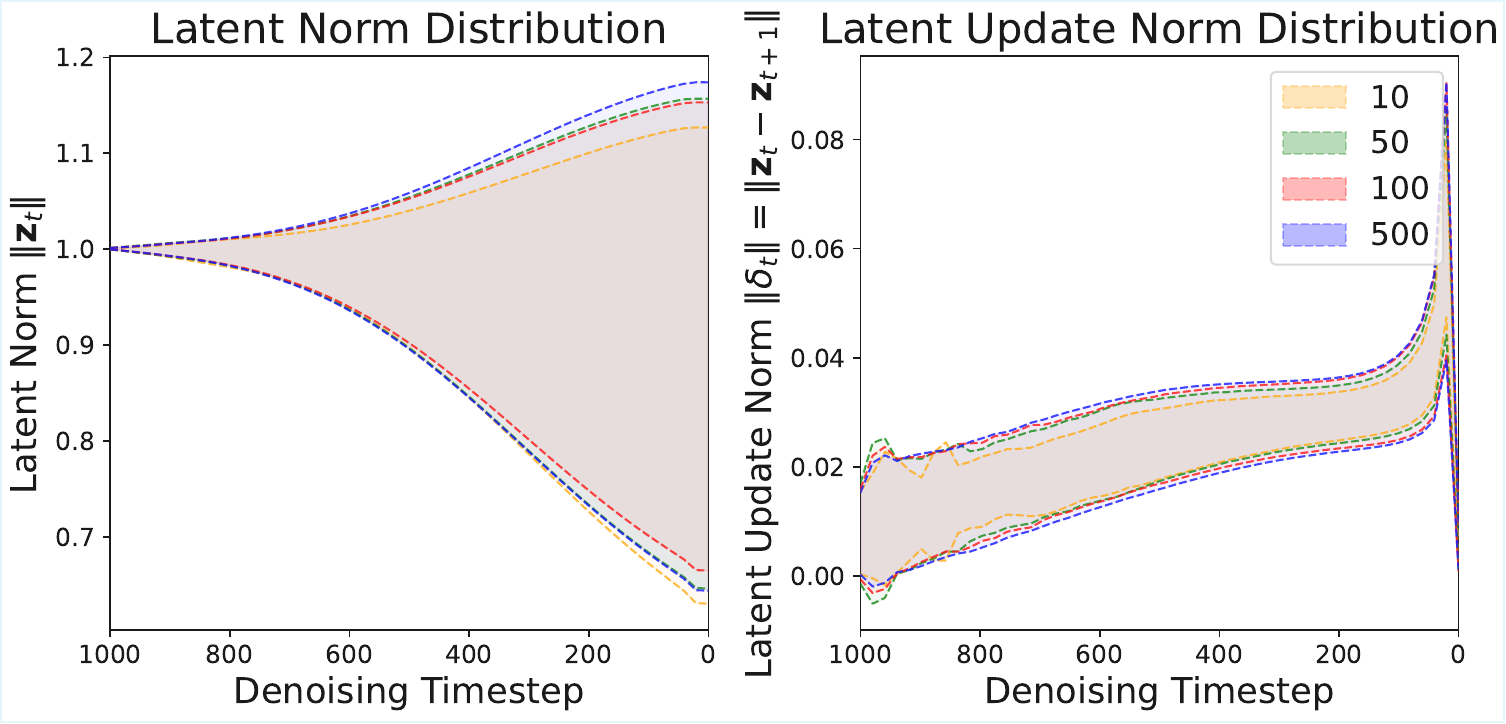}  
\caption{Similar stability regions by different numbers of reference prompts.}
\label{fig:sample_size_stability_regions}
\end{figure}

\begin{figure}[t]
\centering
\includegraphics[width=\linewidth]{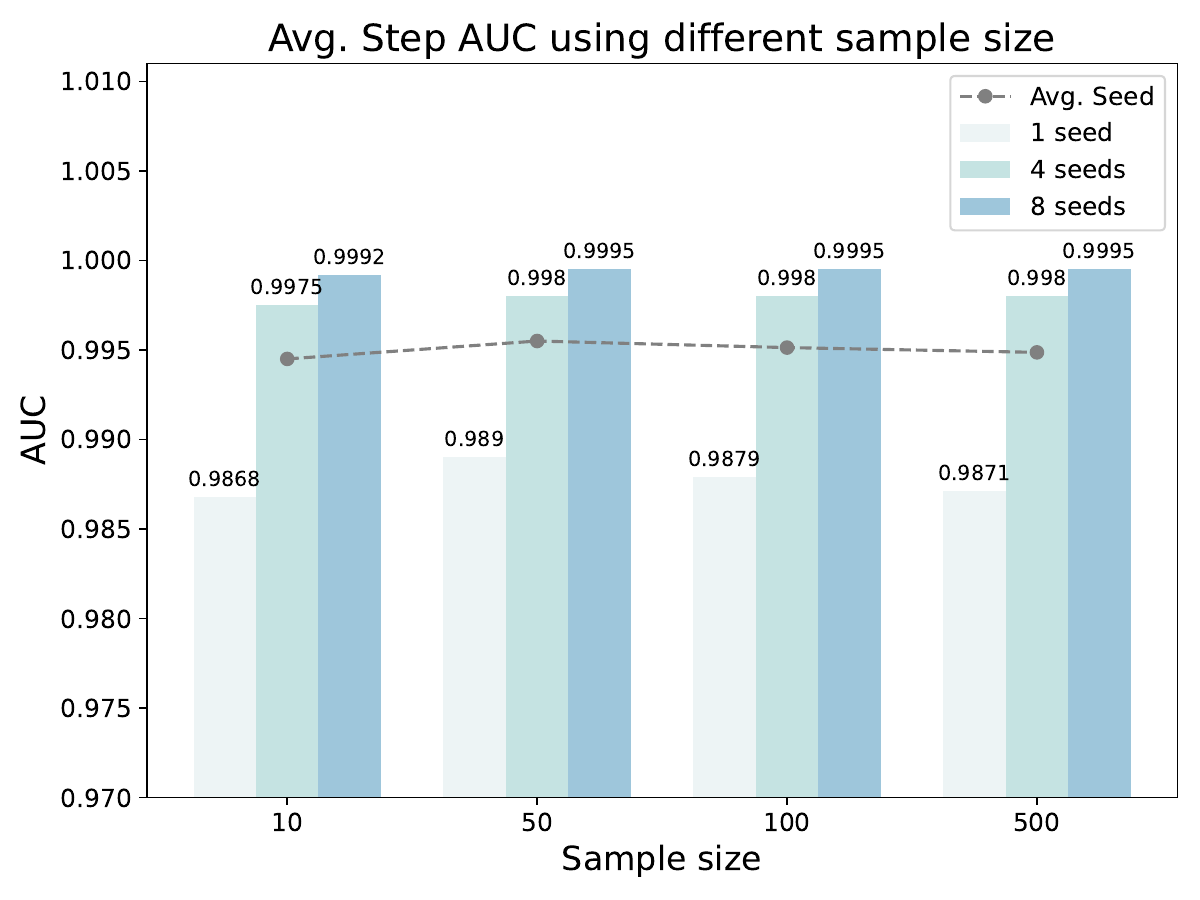}  
\caption{Detection AUC using different numbers of reference prompts.}
\label{fig:auc}
\end{figure}

\begin{figure}[t]
\centering
\includegraphics[width=\linewidth]{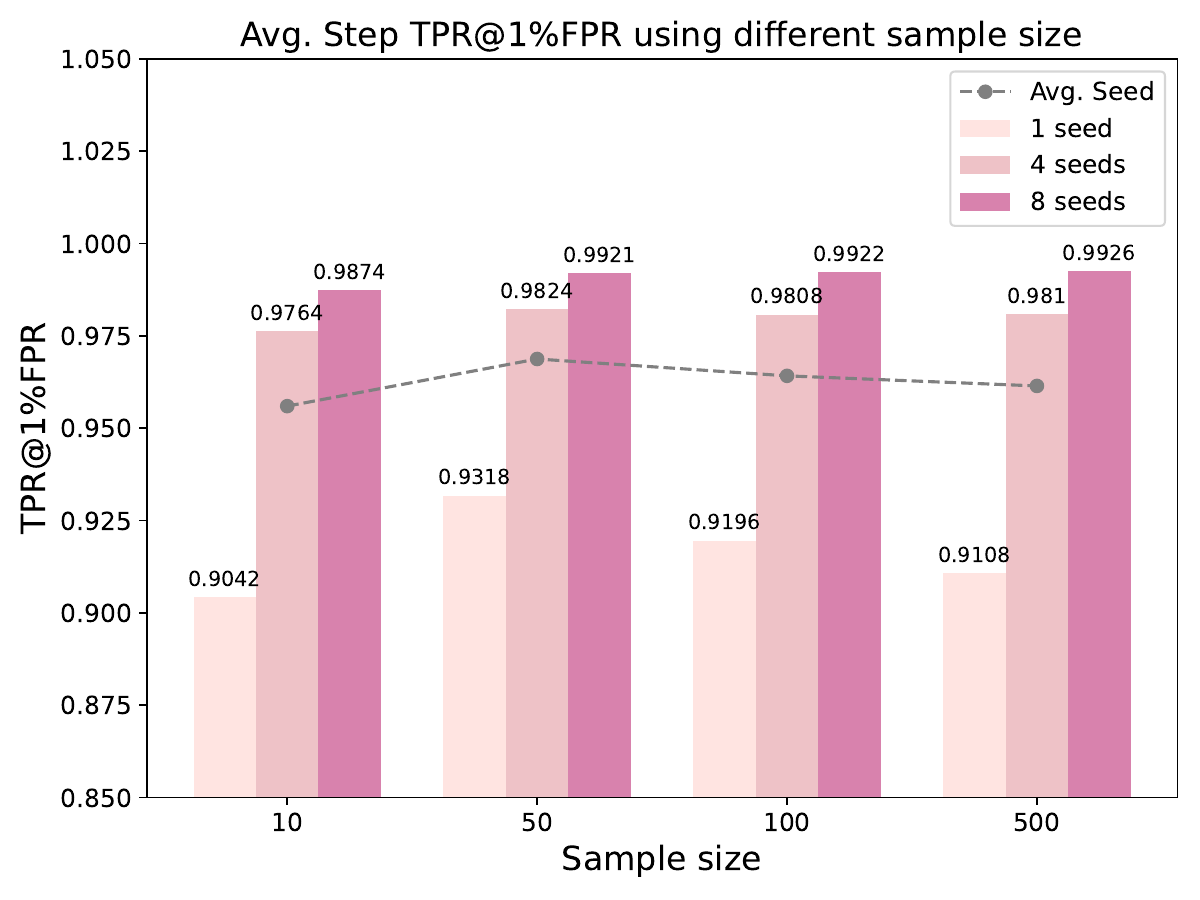}  
\caption{Detection TPR@1\%FPR using different numbers of reference prompts.}
\label{fig:tpr_at_1_fpr}
\end{figure}

\subsection{Memorization Detection on SD 1.5 and 2.1}  \label{sec:more_detection_results_apdx}
From Table~\ref{tab:detect_sd1.5_main} and~\ref{tab:detect_sd2.1_main}, we show that our method exhibits similar patterns of strong detection performance on SD 1.5 and SD 2.1 as on SD 1.4. 
This demonstrates the robustness and generalizability of our approach across different versions of SD models.

\subsection{Memorization Mitigation on Pretrained SD 1.4 and Finetuned SD 1.4}  \label{sec:more_mitigation_results_apdx}
Table~\ref{tab:mitigation_results_sd1.4_0209} and~\ref{tab:mitigation_ft_0209} present the mitigation results on pretrained SD 1.4 and finetuned SD 1.4, respectively.
From Table~\ref{tab:mitigation_ft_0209}, we show that our method effectively mitigates memorization artifacts in the finetuned SD 1.4 model, while maintaining high image quality and semantic alignment.
Note that the memorization induced by finetuning is less stronger than that in pretrained models. Baselines still struggle to mitigate memorization because they fail to detect memorization accurately when memorized generations become hard to recognize in this setting.
Jeon et al. achieve decent performance in this setting, likely because the weaker memorization induced by finetuning is more easily mitigated by their local curvature-aware optimization.

In addition, we point out that previous research usually does not consider the realistic detect-and-mitigate scenario in our work. Instead, they apply mitigation to all generations without detection in advance. 
Therefore, they may exhibit an illusion of mitigation effectiveness. 

\begin{table}[t]
\centering
\caption{Memorization mitigation results on SD 1.4.}
\resizebox{\linewidth}{!}{
\begin{tabular}{lcccc}
\toprule
\multicolumn{1}{c}{Method} & \multicolumn{1}{c}{Mem Rate↓} & \multicolumn{1}{c}{SSCD↓}    & \multicolumn{1}{c}{{\color[HTML]{373C43} CLIP score↑}} & \multicolumn{1}{c}{FID↓} \\ \midrule
No Mitigation              & 0.574                         & 0.519                        & \textbf{32.327}                                        & 198.044                  \\
RTA~\cite{somepalli2023understanding}                       & {0.300}  & {0.316} & {30.773}                          & 197.075                  \\
Ren et al.~\cite{ren2024unveiling}                 & 0.292                         & 0.2651                       & 31.651                                                 & 186.738                  \\
Wen et al.~\cite{wen2024detecting}                 & {0.462}  & {0.442} & {31.706}                          & 172.406                  \\
Jeon et al.~\cite{jeon2025understanding}                & 0.348                         & 0.394                        & 31.775                                                 & 185.689                  \\
Jain et al.~\cite{jain2025classifier}                & 0.094                         & 0.146                        & 28.512                                                 & \textbf{144.344}         \\
Ours                       & \textbf{0.000}                & \textbf{0.074}               & 29.025                                                 & 157.805                  \\ \bottomrule
\end{tabular}
}
\label{tab:mitigation_results_sd1.4_0209}
\end{table}
\begin{table}[t]
\centering
\caption{Memorization mitigation results on finetuned SD 1.4.}
\resizebox{\linewidth}{!}{
\begin{tabular}{lcccc}
\toprule
\multicolumn{1}{c}{Method} & Mem Rate↓      & SSCD↓          & {CLIP score↑} & FID↓                                                            \\ \midrule
No Mitigation              & 0.545          & 0.479          & 31.525                             & {193.610}          \\
RTA~\cite{somepalli2023understanding}                        & 0.275          & 0.321          & 30.659                             & {143.626}          \\
Ren et al.~\cite{ren2024unveiling}                 & 0.545          & 0.479          & 31.525                             & {193.610}          \\
Wen et al.~\cite{wen2024detecting}                 & 0.457          & 0.447          & 31.458                             & {155.586}          \\
Jeon et al.~\cite{jeon2025understanding}                & \textbf{0.000}          & 0.394          & \textbf{32.258}                    & {184.761}          \\
Jain et al.~\cite{jain2025classifier}                & 0.055          & 0.168          & 28.773                             & {\textbf{143.531}} \\
Ours                       & {0.015} & \textbf{0.137} & 29.937                             & {195.254}          \\ \bottomrule
\end{tabular}
}
\label{tab:mitigation_ft_0209}
\end{table}

\begin{figure}[t]
\centering\textbf{}
{
		\includegraphics[width=\linewidth]{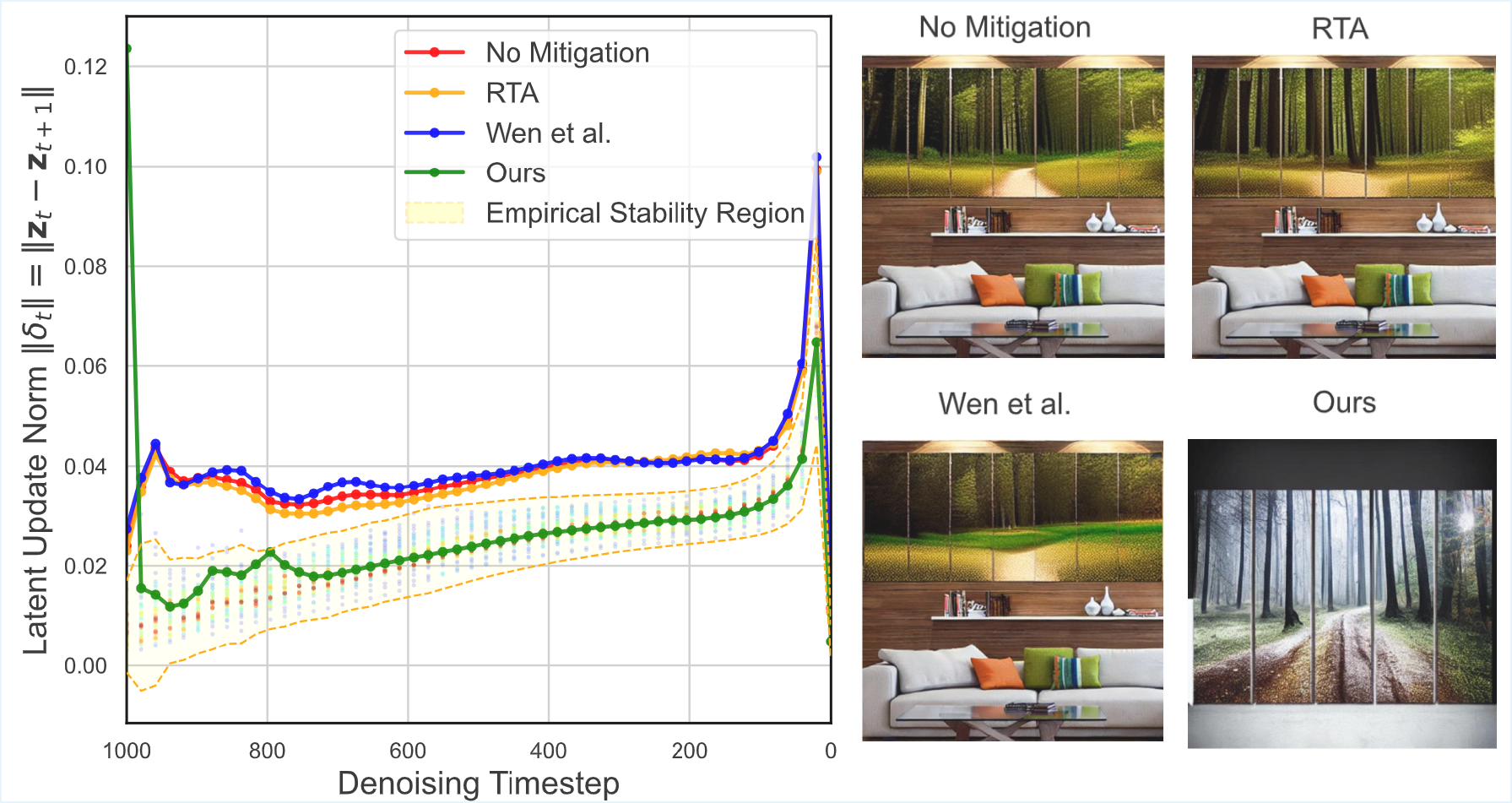}}\\
{
		\includegraphics[width=\linewidth]{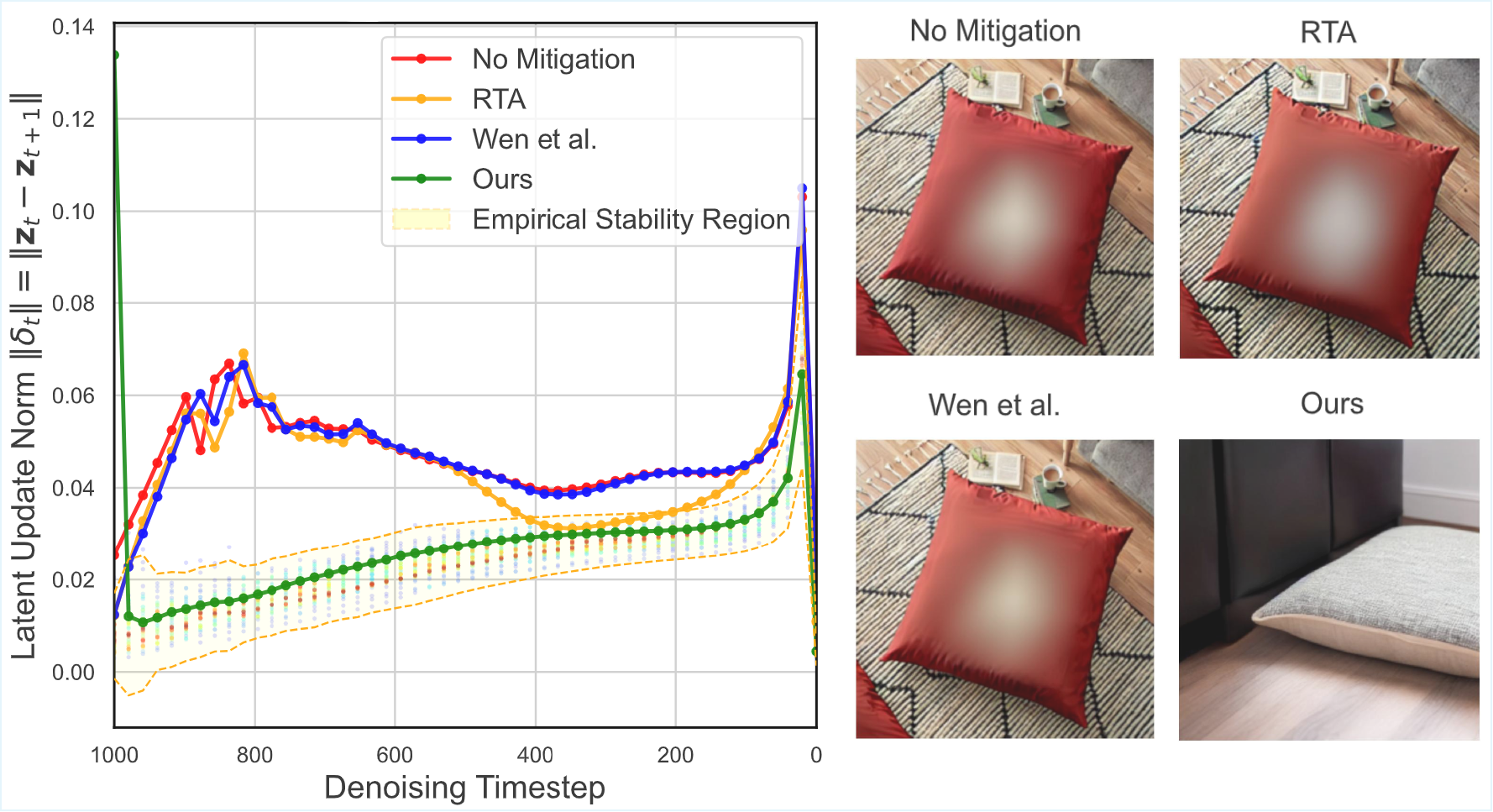}}
\caption{Comparison of latent update trajectories and generated images on SD 1.4 with memorized prompts.}
\label{fig:case_study_appendix}
\end{figure}

\subsection{More Case Studies}  \label{sec:more_case_study_apdx}
{Fig.~\ref{fig:case_study_appendix}} presents additional generation cases, further validating the effectiveness of our method in various scenarios. Compared to the baselines, the baseline generations are quickly drawn towards memorized content in the early steps, resulting in obvious memorization artifacts. In contrast, our method detects and adaptively mitigates memorization at each step, ultimately producing images without memorized content. These cases further demonstrate the necessity of on-the-fly detection and stepwise mitigation for preventing memorization in diffusion models.

\subsection{More Visualizations}
In addition to Fig.~\ref{fig:exp}, which visualizes the mitigation results on pretraind SD 1.4 with memorized prompts, we further present visualizations on finetuned SD 1.4 with memorized prompts (Fig.~\ref{fig:case_ft_mem}). 
These qualitative results further demonstrate that our method can effectively mitigate memorization generations.
In comparison, most baselines either fail to mitigate memorization or degrade the image quality, leading to poor semantic alignment with the original prompts.

\begin{figure*}[t]
\centering
\includegraphics[width=\linewidth]{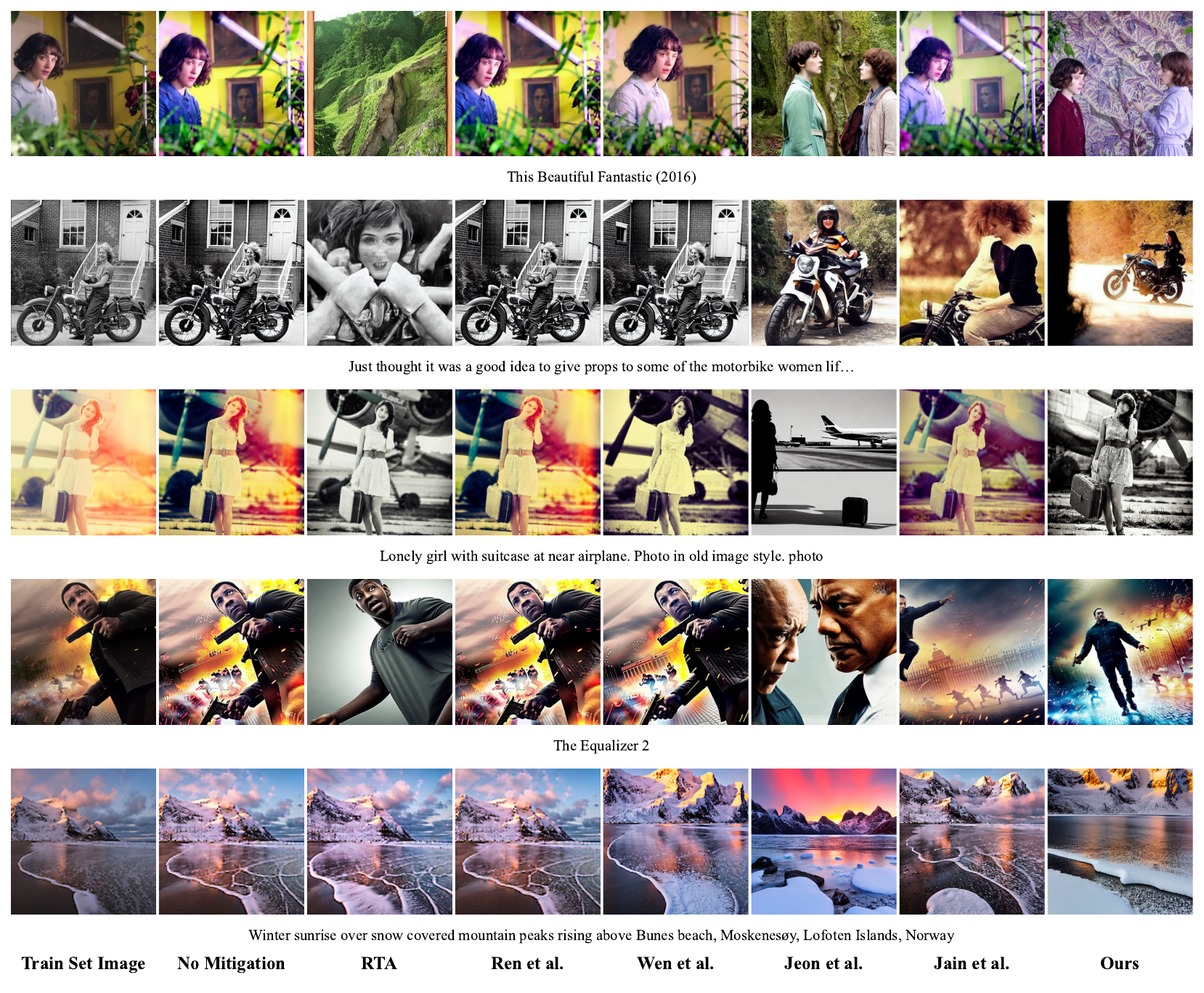}  
\caption{Comparison with baselines on \emph{memorized} prompts using finetuned SD 1.4.}
\label{fig:case_ft_mem}
\end{figure*}

\section{Prompts Used in Visualization}  \label{sec:prompts_used_in_visualization}

\subsection{Fig.~\ref{fig:intro}}
\begin{itemize}
    \item \textbf{Memorized - PNDM/DDIM:} \textit{The No Limits Business Woman Podcast}
    \item \textbf{Non-Memorized:} \textit{this is a pizza that is sliced up in pieces}
\end{itemize}

\subsection{Fig.~\ref{fig:method}}
\begin{itemize}
    \item \textbf{Strong Memorization:} \textit{Mothers influence on her young hippo}
    \item \textbf{Mild Memorization:} \textit{Full body U-Zip main opening - Full body U-Zip main opening on front of bag for easy unloading when you get to camp}
    \item \textbf{Non-Memorization:} \textit{Biobag BioBag Dog Waste Bags On a Roll - Case of 12 - 45 Count HGR 1268796}
\end{itemize}

\subsection{Fig.~\ref{fig:case_study}}
\begin{itemize}
    \item \textit{Foyer painted in HIGH TIDE}
\end{itemize}

\subsection{Fig.~\ref{fig:case_pndm} and~\ref{fig:case_ddim}}
\begin{itemize}
    \item \textbf{Strong Memorization:} \textit{Rambo 5 und Rocky Spin-Off - Sylvester Stallone gibt Updates}
    \item \textbf{Mild Memorization:} \textit{Aero 31-984040GRN 31 Series 13x8 Wheel, Spun, 4 on 4 BP, 4 Inch BS}
    \item \textbf{Non-Memorization:} \textit{Desert nomads with neon tents}
\end{itemize}

\subsection{Fig.~\ref{fig:case_study_appendix}}
\begin{itemize}
    \item \textbf{Top:} \textit{Designart Forest Road In Thick Woods Modern ForestWrapped Canvas Art - 5 Panels}
    \item \textbf{Bottom:} \textit{Meditation Floor Pillow}
\end{itemize}

\end{document}